\documentclass{article} 
\usepackage{arxiv}
\usepackage{tikz}
\usetikzlibrary{arrows.meta}
\usepackage[english]{babel}

\usepackage{amsmath}
\usepackage{graphicx}
\usepackage{subcaption}
\usepackage{algorithmic}
\usepackage{xspace}

\usepackage{booktabs}

\usepackage{graphicx}
\usepackage[colorlinks=true, allcolors=blue]{hyperref}

\newcommand{\MinLag}{\ensuremath{\mathit{min}}}
\newcommand{\MaxLag}{\ensuremath{\mathit{max}}}
\newcommand{\Makespan}{\ensuremath{\mathit{Makespan}}}

\newcommand{\EndOf}{\ensuremath{\mathrm{end}}}
\newcommand{\StartOf}{\ensuremath{\mathrm{start}}}

\newcommand{\BACCHUS}{\ensuremath{\mathsf{BACCHUS}}}
\newcommand{\soru}{\ensuremath{\mathsf{SORU}}}

\newcommand{\soruh}{\mbox{\ensuremath{\mathsf{SORU}}-\ensuremath{\mathsf{H}}}}
\newcommand{\stnu}{\ensuremath{\mathsf{stnu}}}
\newcommand{\react}{\ensuremath{\mathsf{react}}}
\newcommand{\reactive}{\ensuremath{\mathsf{reactive}}}
\newcommand{\pro}{\ensuremath{\mathsf{pro}}}
\newcommand{\proactive}{\ensuremath{\mathsf{proactive}}}
\newcommand{\saa}{\ensuremath{\mathsf{proactive}_{\mathsf{SAA}}}}
\newcommand{\proheur}{\ensuremath{\mathsf{proactive}_{\mathsf{0.9}}}}
\newcommand{\saashort}{\ensuremath{\mathsf{pro}_{\mathsf{SAA}}}}
\newcommand{\RCPSPmax}{RCPSP/\allowbreak{}max\xspace}
\newcommand{\SRCPSPmax}{SRCPSP/\allowbreak{}max\xspace}

\usepackage{amsmath}

\newcommand{\Start}{\ensuremath{\mathit{start}}}
\newcommand{\End} {\ensuremath{\mathit{end}}}
\newcommand{\LB}{\ensuremath{\mathit{LB}}}
\newcommand{\UB}{\ensuremath{\mathit{UB}}}
\usepackage{amsthm}
\usepackage{multirow}
\newtheorem{prop}{Proposition}
\usepackage{authblk} 

\title{Proactive and Reactive Constraint Programming for Stochastic Project Scheduling with Maximal Time-Lags}

\author[1,2]{Kim van den Houten}
\author[1]{Léon Planken}
\author[3]{Esteban Freydell}
\author[1]{David M.J. Tax}
\author[1]{Mathijs de Weerdt}
\affil[1]{Delft University of Technology, Delft, The Netherlands}
\affil[2]{\texttt{k.c.vandenhouten@tudelft.nl}}
\affil[3]{DSM-Firmenich, Delft, The Netherlands}

\begin{document}

\maketitle

\begin{abstract}
This study investigates scheduling strategies for the stochastic resource-constrained project scheduling problem with maximal time lags (\SRCPSPmax)). Recent advances in Constraint Programming (CP) and Temporal Networks have re-invoked interest in evaluating the advantages and drawbacks of various proactive and reactive scheduling methods. First, we present a new, CP-based fully proactive method. Second, we show how a reactive approach can be constructed using an online rescheduling procedure. A third contribution is based on partial order schedules and uses Simple Temporal Networks with Uncertainty (STNUs). Our analysis shows that the STNU-based algorithm performs best in terms of solution quality, while also showing good relative computation time. 
\end{abstract}

 \begin{itemize}
     \item Code \url{github.com/kimvandenhouten/AAAI25_SRCPSPmax}
    \item Official publication [to be published]: \url{AAAI25_ proceedings}
 \end{itemize}

\section{Introduction}
In real-world scheduling applications, durations of activities are often stochastic, for example, due to the inherent stochastic nature of processes in biomanufacturing. At the same time, hard constraints must be satisfied: e.g. once fermentation starts, a cooling procedure must start at least (minimal time lag) 10 and at most (maximal time lag) 30~minutes later. The combination of maximal time lags and stochastic durations is especially tricky: a delay in duration can cause a violation of a maximal time lag when a resource becomes available later than expected. Such constraints are reflected in the Stochastic Resource-Constrained Project Scheduling Problem with Time Lags (\SRCPSPmax). This problem has been an important focus of research due to its practical relevance and the computational challenge it presents, as finding a feasible solution is NP-hard \cite{bartusch1988scheduling}.

Broadly speaking, there are two main schools of thought regarding solution approaches for stochastic scheduling in the literature: 1) proactive scheduling and 2) reactive scheduling. The main goal of proactive scheduling is to find a robust schedule offline, whereas reactive approaches adapt to uncertainties online.  Proactive and reactive approaches can be considered as opposite ends of a spectrum. Both in practice and literature, it is often observed that methods are hybrid, such as earlier work on the \SRCPSPmax. 

Hybrid approaches appear for example in the form of a \textit{partial order schedule} (POS), which is a temporally flexible schedule in which resource feasibility is guaranteed \cite{policella2004generating}. The state-of-the-art POS approach for \SRCPSPmax is the algorithm $\BACCHUS$ \cite{fu2016robustpos}, although the comparison provided by the authors themselves shows that their earlier proactive method $\soruh$ \cite{fu2016proactive} performs better. Partial order schedules are often complemented with a temporal network \cite{pos2009lombardi}, in which time points (nodes) are modeled together with temporal constraints (edges). Recent advances in temporal networks with uncertainties \cite{hunsberger2024foundations} pave the way for improvements in POS approaches for \SRCPSPmax. 

State-of-the-art methods like $\BACCHUS$ and $\soruh$ use Mixed Integer Programming (MIP), whereas Constraint Programming (CP), especially with interval variables \cite{laborie2015intervals}, has become a powerful alternative for scheduling. This is evidenced by a CP model for resource-constrained project scheduling provided by \cite{laborie2018ibmscheduling}. Additionally, a recent comparison between solvers for CP and MIP demonstrated CP Optimizer's superiority over CPLEX across a broad set of benchmark scheduling problems \cite{mipvscp2023}. These advances, however, have not yet been explored for \SRCPSPmax, despite potential applications. CP can be used for finding robust proactive schedules or within a reactive approach with rescheduling during execution, which has been viewed as too computationally heavy \cite{predictive-reactive2008vonder}. 

A proper benchmarking paper in which a statistical analysis is performed on the results of the different methods for \SRCPSPmax is lacking. Comparing different stochastic scheduling techniques for this problem should be done carefully. Due to the maximal time lags and stochastic durations, methods can fail by violating resource or precedence constraints. Therefore, not only the solution quality and computation time but also the feasibility should be taken into account. \cite{fu2016robustpos} even criticize their own evaluation metric and indicate that future work should provide a more extensive comparison. We argue that statistical tests dealing with infeasibilities correctly are needed to compare different methods properly.

This paper proposes new versions of a proactive approach and a hybrid approach using the latest advances in Constraint Programming \cite{laborie2018ibmscheduling} and Temporal Networks \cite{hunsberger2024foundations}. A new reactive approach with complete rescheduling, which has not been investigated before as far as we know, is included in the comparison too. 
We target the existing research gap due to the lack of a clear and informative comparison among various methods, examining aspects such as infeasibility, solution quality, and computation time (offline and online). In doing so, this work provides a practical guide for researchers and industrial schedulers when selecting from various methods for their specific use cases.

\section{The Scheduling Problem}
The \RCPSPmax problem \cite{Kolisch1996PSPLIBA} is defined as a set of activities~$J$ and a set of resources~$R$, where each activity $j \in J $ has a duration~$d_j$ and requires from each resource $r \in R$ a certain amount indicated by~$r_{r,j}$. A solution to the \RCPSPmax is a start-time assignment~$s_j$ for each activity, such that the constraints are satisfied: (i) at any time the number of resources used cannot exceed the max resource capacity~$c_r$; and (ii) for some pairs~$(i, j)$ of activities, precedence constraints are defined as minimal or maximal time lags between the start time of~$i$ and the start time of~$j$. The goal is to minimize the makespan.

We use the following example from \cite{schutt2013solving} of a simple \RCPSPmax instance.
There are five activities: \( a \), \( b \), \( c \), \( d \), and~\( e \), with start times \( s_a \), \( s_b \), \( s_c \), \( s_d \), and~\( s_e \). The activities have durations of 2, 5, 3, 1, and 2~time units, respectively, and resource requirements of 3, 2, 1, 2, and~2 on a single resource with a capacity of $c_r=4$. The precedence relations are as follows: activity~\( b \) starts at least 2 time units after \( a \) starts, activity \( b \) starts at least 1 time unit before \( c \) starts, activity \( c \) cannot start later than 6 time units after \( a \) starts, and activity \( d \) starts exactly 3 time units before \( e \) starts. These constraints can be visualized in a project graph (Figure~\ref{fig:example_graph_rcpsp_max}). A feasible solution to this problem is $\{s_a=1, s_b=3, s_c=4, s_d=0, s_e=3\}$, for which the Gantt chart is visualized in Figure~\ref{fig:chaining}.

\begin{figure}[htb!]
    \centering
\begin{tikzpicture}
\begin{scope}[every node/.style={circle,thick,draw}]
    \node (a) at (1.5,0) {a};
    \node (b) at (3 ,0) {b};
    \node (c) at (4.5,0) {c};
    \node (d) at (2,1) {d};
    \node (e) at (4,1) {e};
\end{scope}

\begin{scope}[every node/.style={thick,draw}]
    \node (source) at (0,0.5) {};
    \node (sink) at (6,0.5) {};
\end{scope}

\begin{scope}[>={Stealth[black]},
              every node/.style={},
              every edge/.style={draw=black,very thick}]
    \path [->] (source) edge node[midway, above] {$0$} (d);
    \path [->] (source) edge node[midway, below] {$0$} (a);
    \path [->] (d) edge node[midway, below] {$3$} (e);
    \path [->] (e) edge[bend right=30] node[midway, above] {$-3$} (d);
    \path [->] (e) edge node[midway, above] {$2$} (sink);
    \path [->] (a) edge node[midway, above] {$2$} (b);
    \path [->] (b) edge node[midway, above] {$1$} (c); 
    \path [->] (c) edge[bend left=30] node[midway, below] {$-6$} (a); 
 \path [->] (c) edge node[midway, below] {$2$} (sink); 
\end{scope}
\end{tikzpicture}
   \caption{Example project graph \cite{schutt2013solving}.}
    \label{fig:example_graph_rcpsp_max}
\end{figure}

 A special property of \RCPSPmax is the following:
 \begin{prop}Suppose we are given a problem instance~$1$ with durations~$d^1$, resource requirements~$r^1$, and capacity~$c^1$, and let $s^1$ be a feasible schedule for this instance. Suppose now that we transform instance~$1$ into instance~$2$, where all parameters stay equal except that one or more of the activity durations~$d^2$ are shorter than the durations~$d^1$, so $\forall \ j \in J: d^1_j \leq d^2_j$. Then, $s^1$ is also feasible for instance~$2$.  \label{prop:smaller_durations}
 \end{prop}
 \begin{proof}
Schedule $s^1$ is still precedence feasible because the start times did not change and the precedence constraints are defined from start to start (they are deterministic). Schedule~$s^1$ is resource feasible for $d^1$. Since $d^2$ is strictly smaller than~$d^1$, the resource usage over time can only be smaller than for instance~$1$, and thus it will not exceed the capacity, and schedule~$s^1$ is also resource feasible for~$d^2$. 
\end{proof}

This property can be helpful when reusing start times for a mutated instance, for example, because of stochastic activity durations. The Stochastic \RCPSPmax (\SRCPSPmax) is an extension in which the durations follow a stochastic distribution \cite{fu2012robustls}. We define the problem as follows: the activity durations are independent random variables indicated by $d_j$ representing the duration of activity j. We assume that each $d_j$ follows a discrete uniform distribution, $d_j \sim$ DiscreteUniform($lb_j, ub_j$) where $lb_j$, $ub_j$ denote the minimum and maximum possible durations for activity $j$. Each duration becomes available when an activity finishes. 

In general, comparing methods for this problem is not straightforward as while comparing method A and B, it could be that only method A obtains a feasible solution. Looking at double hits (instances solved by both methods) is a reasonable choice, but neglects that potentially one of the two methods solved more problem instances than the other. Earlier work on this stochastic variant introduces the $\alpha$-robust makespan \cite{fu2012robustls}, that is the expected makespan value for a scheduling strategy for which the probability that the schedule is feasible is at least $1-\alpha$, as a metric for comparing different methods but also indicated the limitations of this metric.

\section{Background and Related Work}
In this section, we discuss existing approaches for \SRCPSPmax. Furthermore, we introduce all concepts that are needed to understand our scheduling methods that are presented in Section \ref{sec:methods}. We introduce proactive techniques based on Sample Average Approximation in Section \ref{sec:background_proactive}. We explain partial order schedules and temporal networks in Section \ref{sec:background_pos}. Finally, Section \ref{sec:related_work} gives an overview of related work on benchmarking scheduling methods for \SRCPSPmax.

\subsection{Proactive Scheduling}\label{sec:background_proactive}

Proactive scheduling methods aim to find a robust schedule offline by taking information about the uncertainty into account \cite{herroelen2002survey}. In this section, we explain a core technique used in proactive methods: Sample Average Approximation (SAA). Furthermore, we discuss the state-of-the-art proactive methods for \SRCPSPmax. 

\subsubsection{Sample Average Approximation}
A common method for handling discrete optimization under uncertainty is the Sample Average Approximation (SAA) approach \cite{kleywegt2002sample}. Samples are drawn from stochastic distributions and added as scenarios to a stochastic programming formulation. The solver then seeks a solution feasible for all scenarios while optimizing the average objective. However, adding more samples to the SAA increases the number of constraints and variables, significantly raising the solution time.

\subsubsection{SORU and SORU-H}
The most recent proactive method on \SRCPSPmax is proposed by \cite{fu2016proactive} and is recognized as the state of the art. The authors present the algorithm $\soru$, an SAA approach to the scheduling problem that relies on Mixed Integer Programming (MIP) and aims to minimize the $\alpha$-robust makespan. It can be summarized as follows:
(i)~a selection of samples is used to set up the SAA; (ii)~the model seeks a start time vector~$s$ such that the minimal and maximal time lags of precedence constraints are satisfied; (iii)~it allows for $\alpha\%$ of the scenarios to be resource infeasible; (iv)~it minimizes the sample average makespan. 

Since $\soru$ is computationally expensive, the authors propose a heuristic version, dubbed $\soruh$. Instead of a set of samples, one summarizing sample is used, which represents a quantile of the distribution. Note that because of Proposition~\ref{prop:smaller_durations}, this heuristic approximates the $\alpha$-robust makespan. At the same time, it is much cheaper to compute because typically the runtime increases for a larger sample size in SAA, as shown by \cite{fu2016proactive}.

Since nowadays CP is the state of the art for deterministic project scheduling \cite{mipvscp2023}, we re-investigate SAA approaches for \SRCPSPmax with~CP and compare this CP-based proactive approach to reactive approaches.

\subsection{Partial Order Scheduling}\label{sec:background_pos} 
In this section, we discuss partial order scheduling approaches, which can be seen as a reactive-proactive hybrid. Partial order schedules have been used in the majority of the contributions to \SRCPSPmax.

\subsubsection{Constructing Ordering Constraints}
A \textit{partial order schedule} (POS) can be seen as a collection of schedules that ensure resource feasibility, but maintain temporal flexibility. A POS is defined as a graph where nodes represent activities, and edges temporal constraints between them. There are several methods to derive a POS. In the original paper by \cite{policella2004generating}, two approaches are outlined to construct the ordering constraints between activities, either analyzing the resource profile to avoid all possible resource conflicts, based on \textit{Minimal Critical Sets} (MCS) \cite{lgelmund1983algorithmic}, or using a single-point solution together with a chaining procedure to construct \textit{resource chains} (see Figure~\ref{fig:chaining}). They find that using the single-point solution together with chaining seems both simple and most effective. Subsequent work explored alternative MCS-based approaches  \cite{pos2009lombardi,pos2013lombardi} and chaining heuristics \cite{fu2012robustls}. 

\begin{figure}[htb!]
    \centering
    \includegraphics[width=0.4\textwidth]{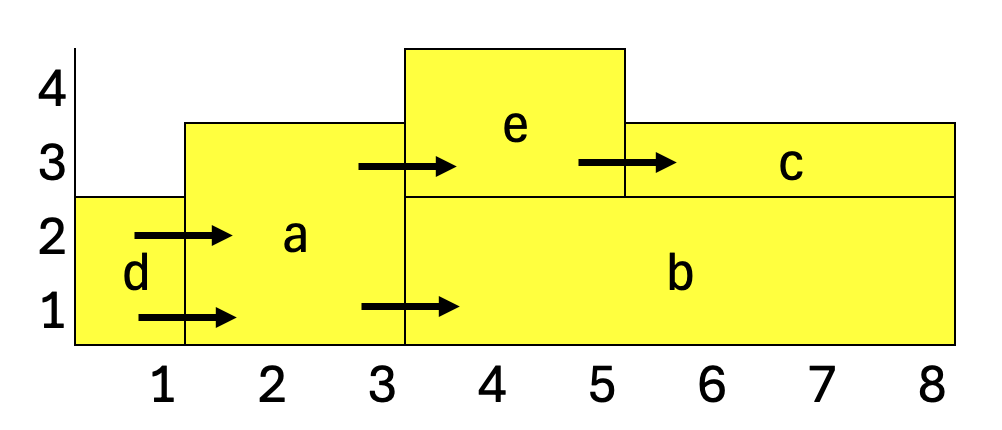}
    \caption{Example Gantt chart (figure adjusted from \cite{fu2012robustls}). The X-axis is time, and Y-axis shows resource demand. 
    \textbf{The arrows} indicate an example of generating a POS with chaining, starting from a fixed solution schedule.}
    \label{fig:chaining}
\end{figure}

\subsubsection{Temporal Networks}
Partial order schedules are often complemented with a temporal model to reason over the temporal constraints. Most POS approaches rely on a Simple Temporal Network (STN), which is a graph consisting of time points (nodes) and temporal difference constraints (edges). The Simple Temporal Network with Uncertainty (STNU) extends the STN by introducing \textit{contingent links}. The duration of these contingent constraints can only be observed, while for regular constraints it can be determined during execution.  The works by \cite{pos2009lombardi} and \cite{pos2013lombardi} on POS are the only ones that we know of using STNUs as their temporal model.
They introduce nodes for the start and the end of each activity with duration constraints between them, connecting respective start nodes with edges for the precedence constraints.

\subsubsection{Execution Strategies}
An STNU is \textit{dynamically controllable} (DC) if there is a strategy to determine execution times for all controllable (non-contingent) time points which ensures that all temporal constraints are met, regardless of the outcomes of the contingent links. Some DC-checking algorithms generate new so-called wait edges to make the network~DC \cite{morris2014dynamic}.

The DC STNU with wait edges is referred to as an Extended STNU (ESTNU) and is input to a Real-Time Execution Algorithm. Such an algorithm is the online component that transforms an STNU into a schedule (i.e. an execution time for each node), given the observations for the contingent nodes. An algorithm specifically tailored to ESTNUs is RTE$^*$ \cite{hunsberger2024foundations,posenato2022cstnu}.

As far as we know, the DC-checking and RTE$^*$ algorithms have not been applied to \SRCPSPmax, despite their efficiency. \cite{pos2009lombardi,pos2013lombardi} used constraint propagation for DC-checking, but do not use RTE$^*$. Other works on POS for \SRCPSPmax that do not use STNUs and DC-checking risk violations of minimal or maximal time-lags during execution \cite{pol2007pos,fu2016robustpos}. Thus, we conclude that there is a research gap in applying the developments in STNU literature to \SRCPSPmax.

\subsection{Benchmarking Approaches}\label{sec:related_work}
Benchmarking procedures for comparing scheduling methods are inconsistent in the literature, with varying problem sets and comparison methods. Some studies used industrial scheduling instances \cite{pos2009lombardi,pos2013lombardi}, while the majority \cite{policella2004generating, fu2012robustls, fu2016proactive, fu2016robustpos} relied on PSPlib \cite{Kolisch1996PSPLIBA} instances, which are deterministic and transformed into stochastic versions with noise. Research typically focused on instances with 10, 20, and 30 activities (j10-j30). Different studies assessed varying metrics, such as schedule flexibility and robustness \cite{policella2004generating,pol2007pos}, solver performance \cite{pos2009lombardi,pos2013lombardi}, or the $\alpha$-robust makespan \cite{fu2012robustls,fu2016proactive,fu2016robustpos}. However, no comprehensive benchmarking paper exists that evaluates both solution quality and computation time while also correctly accounting for infeasibilities. The main challenge is that not all instances can be solved by all methods, making it difficult to directly compare the distributions of solution quality or other metrics. We take inspiration from \cite{icaps}, who compare different planners that can fail, providing a framework for comparing solution quality and speed while also correctly considering these failures. 

\section{Scheduling Methods}\label{sec:methods}
This section outlines the proposed methods. We explain how to use CP for these scheduling problems so far dominated by MIP approaches. Note that in the Technical Appendix \cite{houten2024proactive}, we include a comparison between a CP and a MIP approach for \RCPSPmax  (CPOptimizer and CPLEX; \cite{cplex}) demonstrating that CP outperforms MIP, which is in line with the literature. 

First, we present a deterministic CP model for \RCPSPmax in Section \ref{sec:det_cp}.
The new, stochastic methods for \SRCPSPmax are proposed in Section \ref{sec:new_methods}. We explain the statistical tests for performing pairwise comparisons of these new methods that lead to partial orderings based on solution quality and computation time in Section \ref{sec:stat}. 

\subsection{Constraint Programming for \RCPSPmax}\label{sec:det_cp}
The CP model is: 
\begin{subequations}
\begin{align*}
   \text{Min } \Makespan   \quad  \text{subject to} \span  \\
 \max\{\EndOf(x_j)\} \leq \Makespan; &   \ \forall j \in J \\
\StartOf(x_i) \geq \MinLag_{j,i} + \StartOf(x_j); & \forall j \in J \ \forall i \in S_j  \\
\StartOf(x_i) \leq \MaxLag_{j,i} + \StartOf(x_j); & \ \forall j \in J \ \forall i \in S_j  \\
 \sum_{j \in J} \mathrm{Pulse}(x_j, r_{r,j}) \leq c_r;  & \  \forall r \in  R \\
x_j: \mathrm{IntervalVar}(J, d_j); &  \ \forall j \in J 
\end{align*}
\end{subequations}
For the deterministic \RCPSPmax, we use the modern interval constraints from the IBM CP optimizer \cite{cplex}.
In the equations above,
we use IBM's syntax and modify the RCPSP example from \cite{laborie2018ibmscheduling} to \RCPSPmax. We use the earlier introduced nomenclature together with the minimal time lags $min_{j, i}$ and maximal time lags $max_{j,i}$ that are the temporal differences between start times of activities $j$ and $i$ if $i$ is a successor of $j$. We introduce the decision variable $x_j$ as the interval variable for activity $j\in J$. The $\mathrm{Pulse}$ function generates a cumulative expression over a given interval $x_j$ with a certain value. For a task 
$j$, this value is its resource usage $r_{r,j}$. The aggregated pulse values are constrained so that their sum does not exceed the total available resource capacity $c_r$. 

\subsection{New Methods for \SRCPSPmax}\label{sec:new_methods}
This section introduces three new methods for \SRCPSPmax. The first method is a CP-based version of  a proactive model. Then, we present a novel, fully reactive scheduling approach employing the deterministic model for \RCPSPmax. Finally, we propose an STNU-based approach using CP and POS.
We refer to these approaches as \proactive, \reactive, and \stnu, respectively.
\subsubsection{Proactive Method}\label{sec:}
We outline how to use a scenario-based CP model (instead of MIP) for \SRCPSPmax which we call $\saa$.

For this SAA method, we can reuse the deterministic \RCPSPmax CP model, introduced in Section \ref{sec:det_cp}, but we introduce scenarios. This model is inspired by the MIP version by \cite{fu2016proactive}. A special variant is the SAA with only one sample, for which a $\gamma$-quantile can be used which we call $\proactive_{\gamma}$. If a feasible schedule can be found for the $\gamma$-quantile, this schedule will also be feasible for all duration realizations on the left-hand side of the  $\gamma$-quantile because of Proposition~\ref{prop:smaller_durations}. We provide the SAA model below. We use the same nomenclature as for the deterministic model, but we introduce the notion of scenarios $\omega \in \Omega$, and find a schedule $s_j \ \forall j \in J$ that is feasible for all scenarios it has seen if one exists:
\begin{subequations}
\begin{align*}
   \text{ Min} \frac{1}{|\Omega|} \sum_{\omega \in \Omega} \Makespan({\omega})  \quad \text{subject to} \span \\
 \mathrm{Max_j(\EndOf(x_j^{\omega}))} \leq \Makespan({\omega}); & \  \forall \omega \in \Omega \\
\StartOf(x_i^{\omega}) \geq \MinLag_{j,i} + \StartOf(x_j^{\omega}); &\ \forall j \in J \ \forall i \in S_j \ \forall \omega \in \Omega \\
\StartOf(x_i^{\omega}) \leq \MaxLag_{j,i} + \StartOf(x_j^{\omega}); & \ \forall j \in J \ \forall i \in S_j \ \forall \omega \in \Omega \\
 \sum_{j \in J} \mathrm{Pulse(x_j^{\omega}, r_{r,j}) \leq c_r};  & \ \forall r \in  R \ \forall \omega \in \Omega \\
x_j^{\omega}: \mathrm{IntervalVar(J, y_j^{\omega})}; & \ \forall j \in J \ \forall \omega \in \Omega \\
 s_j = \StartOf(x_j^{\omega}); & \ \forall j \in J \ \forall \omega \in \Omega 
\end{align*}
\end{subequations}

\subsubsection{Reactive Method}
We now present a CP-based fully reactive approach for \SRCPSPmax that we refer to as $\reactive$. 

Fully reactive approaches, involving complete rescheduling by solving a deterministic \RCPSPmax, have been considered impractical due to high computational demands and low schedule stability \cite{predictive-reactive2008vonder}. However, advances in CP for scheduling \cite{laborie2018ibmscheduling, mipvscp2023} mitigate these issues. In the industrial setting where we are applying our work, decision-makers often reschedule their entire future plans when changes occur. Thus, including this reactive approach in our comparison is valuable. To our knowledge, such an approach has not been evaluated before. The outline is:

\begin{itemize}
    \item Start by making an initial schedule with an estimation of the activity durations $\hat{d}$. We can see $\hat{d}$ as a hyperparameter for how conservative the estimation is. For example, using the mean of the distribution could lead to better makespans, but the risk of becoming infeasible for larger duration realizations, while taking the upper bound of the distribution could lead to a very high makespan.  
    \item At every decision moment (when an activity finishes) resolve the deterministic \RCPSPmax while fixing all variables until the current time to reschedule with new information, we again use the estimation $\hat{d}$ for the activities that did not finish yet. Resolving is needed when the finish time of an activity deviates from the estimated finish time. We warm start the solver with the previous solution.
\end{itemize}

\subsubsection{STNU-based Method}

Finally, we present a partial order schedule approach using CP and STNU algorithms \cite{hunsberger2024foundations} which we call $\stnu$. This approach is inspired by many earlier works \cite{pol2007pos,pos2009lombardi,fu2016robustpos}. We use a fixed-solution approach for constructing the ordering constraints. The outline of the approach is:
\begin{enumerate}
    \item We make a fixed-point schedule by solving the deterministic \RCPSPmax with an estimation of the activity durations $\widehat{d}$ and using the chaining procedure, which was explained in Section \ref{sec:background_pos}.
    \item We construct the STNU as follows: 
    \begin{enumerate}
        \item For each activity two nodes are created, representing its start and end.
        \item Contingent links are included between the start of the activity and the end of the activity with $[\LB, \UB]$, where $\LB$ and $\UB$ are the lower and upper bounds of the duration of that activity.
        \item The minimal and maximal time lags are modeled using edges $(\Start_B, -\MinLag, \Start_A)$ and $(\Start_A, \MaxLag, \Start_B)$, where A and B are the preceding and succeeding activity, respectively. These edges correspond to the edges in the precedence graph of the instance, see Figure~\ref{fig:example_graph_rcpsp_max}.
        \item The resource chains that construct the POS are added as additional edges as $(\Start_B, 0, \End_A)$ if activity A precedes activity B. Each arrow in Figure~\ref{fig:chaining} would lead to a resource chain edge.
    \end{enumerate}
\end{enumerate}

The resulting STNU is tested for dynamic controllability (DC), and if the network is DC its extended form (ESTNU) is given to the RTE$^*$ algorithm. Since makespan minimization is of interest, we slightly adjust the algorithm from \cite{hunsberger2024foundations}: instead of selecting an arbitrary executable time point and an arbitrary allowed execution time, we always choose the earliest possible time point at the earliest possible execution time. In the Technical Appendix we include a step-by-step example of the STNU-based method.

\subsection{Statistical Tests for Pairwise Comparison}\label{sec:stat}
A strategy for benchmarking is to provide partial orderings of the scheduling methods for the different metrics of interest (e.g. solution quality, runtime offline, runtime online). A partial ordering can be obtained by executing pairwise comparisons of the methods per problem size and per metric, taking inspiration from \cite{icaps}. 

The Wilcoxon Matched-Pairs Rank-Sum Test (the version by \cite{cureton1967normal}) looks at the ranking of absolute differences and gives insight into which of a pair of methods has consistently better performance than another method.   Infeasible cases can be handled by assigning infinitely bad time and solution quality to these cases, leading to an absolute difference of $\infty$ or $-\infty$ that will be pushed to the highest and lowest rankings.

An alternative test to the Wilcoxon test that is also used by \cite{icaps} is the proportion test (see Test 4 in the book by \cite{100tests}). This test is weaker than the Wilcoxon, but provides at least information about significance in the proportion of wins when the Wilcoxon test shows no significant difference. 
The magnitude test provides more insight into the magnitude differences of the performance metrics. This test is also known as the pairwise t-test on two related samples of scores (see Test 10 in the book by \cite{100tests}). This test can only be performed on double hits because infinitely bad computation time or solution quality will disturb the test.

We provide a detailed explanation of all of the above statistical tests in our Technical Appendix.

 \begin{table*}[htb!]
 \centering
 \scalebox{0.9}{
\fontsize{9}{11}\selectfont
\setlength{\tabcolsep}{1mm} 
\begin{tabular}{llllllll} \toprule
  Test & Legend & $\stnu$-$\react$ & $\stnu$-$\saashort$ & $\stnu$-$\pro_{0.9}$ & $\react$-$\saashort$ & $\react$-$\pro_{0.9}$ & $\saashort$-$\pro_{0.9}$ \\ \midrule
    Wilc. Quality & [n] z (p) & [370] -6.75 (*) & [370] -6.8 (*) & [370] -6.86 (*)  & [370] -2.76 (*) & [370] -8.53 (*) & [370] -7.09 (*) \\  \midrule
   Prop. Quality & [n] prop (p)
    & [277] 0.78 (*) & [276] 0.78 (*) & [278] 0.78(*) & [61] 0.67 (*) & [73] 1.0 (*) & [65] 0.94 (*) \\ \midrule
    Magn. Quality & $[n] \ t \ (p)$ & [330] -11.36 (*)  & [330] -11.39 (*)  & [330] -11.35 (*)  & [370] -2.86 (*)  & [370] -6.73 (*)  & [370] -6.01 (*)  \\ 
    & norm. avg. & $\stnu$: 0.985 & $\stnu$: 0.985 & $\stnu$: 0.984 & $\react$: 1.0 
  & $\react$: 0.999 & $\saashort$: 0.999 \\
     & norm. avg. & $\react$: 1.015 & $\saashort$: 1.015 & $\pro_{0.9}$: 1.016 & $\saashort$: 1.0 & $\pro_{0.9}$: 1.001 & $\pro_{0.9}$: 1.001 \\ \toprule \toprule
      Test &  Legend & $\react$-$\stnu$ & $\react$-$\saashort$ & $\pro_{0.9}$-$\stnu$ & $\pro_{0.9}$-$\saashort$ & $\stnu$-$\saashort$ & \\  \midrule
     Wilc. Offline & [n] z (p) &  [370] -16.67 (*) & [370] -16.67 (*) & [370] -16.67 (*) & [370] -16.67 (*) & [370] -7.26 (*) & \\ \midrule 
     Prop. Offline & [n] prop (p)  & [370] 1.0 (*) & [370] 1.0 (*) & [370] 1.0 (*) & [370] 1.0 (*) & [370] 0.62 (*) \\ \midrule
     Magn. Offline & $[n] \ t \ (p)$ &  [330] -36.15 (*)  & [370] -72.56 (*)  & [330] -36.15 (*)  & [370] -72.56 (*)  & [330] -5.64 (*) &  \\ 
    & norm. avg. &  $\react$: 0.36 & $\react$: 0.24 & $\pro_{0.9}$: 0.36 & $\pro_{0.9}$: 0.24 & $\stnu$: 0.82 & \\
  &  norm. avg.  &  $\stnu$: 1.64 & $\saashort$: 1.76 & $\stnu$: 1.64 & $\saashort$: 1.76 & $\saashort$: 1.18  & \\  \toprule \toprule
     Test &   Legend & $\pro_{0.9}$-$\stnu$ & $\saashort$-$\stnu $& $\pro_{0.9}$-$\react$ & $\saashort$-$\react$ & $\stnu$-$\react$ 
 & \\ \midrule
      Wilc. Online & [n] z (p) &  [370] -16.67 (*) & [370] -16.67 (*) & [370] -16.67 (*) & [370] -16.67 (*) & [370] -9.85 (*) & \\  \midrule
  Prop. Online & [n] prop (p)  & [370] 1.0 (*) & [370] 1.0 (*) & [370] 1.0 (*) & [370] 1.0 (*) & [370] 0.89 (*)  & \\  \midrule
   Magn. Online & $[n] \ t \ (p)$ &  [330] -3.95$ e3$ (*)  & [330] -3.98$ e3$ (*)  & [370] -3.07$e4$ (*)  & [370] -3.14$e4$ (*)  & [330] -142.76 (*) &  \\ 
   & norm. avg. & $\pro_{0.9}$: 0.01 & $\saashort$: 0.01 & $\pro_{0.9}$: 0.0 & $\saashort$: 0.0 & $\stnu$: 0.15 & \\
    & norm. avg. &  $\stnu$: 1.99 & $\stnu$: 1.99 & $\react$: 2.0 & $\react$: 2.0 & $\react$: 1.85 & \\ \bottomrule
\end{tabular}}
\caption{Pairwise test results for ubo50, $\epsilon=1$, covering solution quality and runtime. Metrics include Wilcoxon, proportion, and magnitude tests (Section \ref{sec:stat}). Results: [pairs] z-value (p-value) (* for $p<0.05$), proportion (p-value), and t-stat (p-value) with normalized averages. Additional results for other $\epsilon$, j10–30, ubo50, and ub0100 are in the Technical Appendix. Column headers list the compared methods, excluding zero-difference pairs.}\label{tab:test_results}
\end{table*}

\section{Experimental Evaluation}

The goal of the evaluation is to analyze the relative performance of the different proposed scheduling methods. 
\subsection{Data Generation}
We use the j10, j20, j30, ubo50, and ubo100 sets from the PSPlib \cite{Kolisch1996PSPLIBA}. Instead of using all instances from j10 to j30, we select 50 per set and extend previous work \cite{fu2016proactive,fu2016robustpos} by including 50 instances each from the ubo50 and ubo100 sets. Following prior research, we use deterministic durations to set up stochastic distributions with different noise levels, converting deterministic instances into stochastic ones. We use uniform discrete distributions for durations based on the deterministic processing times $d_j$. Specifically, we define the lower bound as $lb_j = \max(1, \text{round}(d_j - \epsilon \cdot \sqrt{d_j}))$ and the upper bound as $ub_j = \text{round}(d_j + \epsilon \cdot \sqrt{d_j})$, where we vary the noise level $\epsilon=\{1, 2\}$. The source and sink nodes always have a deterministic duration of zero. Each evaluation of a method corresponds to one sample from the distribution, and we sample $10$ times for each instance. We excluded the instance samples for which it is impossible to find a feasible schedule (i.e. the deterministic problem with perfect information is infeasible for the given activity duration sample).

\subsection{Tuning of the Methods}
In this section, we highlight the most important observations and outcomes of our tuning results; for further details, see our Technical Appendix.

All CP models are solved with the IBM CP solver \cite{cplex} with default settings. Most of the deterministic \RCPSPmax the j10-j30 instances can be solved within 60 seconds. For, ubo50 and ubo100 instances, we fixed the time limit to 600 seconds.  We tune the classical SAA with multiple scenarios $\saa$ and a heuristic approach $\proactive_{\gamma}$. We used $\gamma=0.9$ for $\proactive_{\gamma}$. For $\saa$, we sampled quantiles at $\gamma \in [0.25, 0.5, 0.75, 0.9]$ and set a time limit of 1800 seconds. The algorithm $\reactive$ uses offline the $\proactive_{\gamma}$ algorithm, for which we fixed $\gamma=0.9$. We investigated the effect of the time limit for rescheduling and set this hyperparameter to 2 seconds. The algorithm $\stnu$ consists of an offline procedure, which comprises building up the network and checking dynamic controllability; if it's not DC, the method fails. We observed that that the choice of the $\gamma$-quantile for the fixed-point schedule significantly affects the ratio of DC networks. We select $\gamma=1$ for our final method $\stnu$, because this setting results in much more DC networks than lower quantiles.

\subsection{Results}\label{sec:results}

We include an analysis of the feasibility ratios of the different methods. A selection of the results of the statistical tests are shown in Table \ref{tab:test_results} (the remaining results are included in the Technical Appendix). We present summarized partial orderings on 1) solution quality, 2) time offline, 3) time online.

\subsubsection{Feasibility Ratio}
We first analyze the obtained feasibility ratios in Table \ref{tab:feasible_combined}. For $\epsilon=1$, the feasibility ratios of $\saa$, $\proactive_{0.9}$, and $\reactive$ are similar for instance sets j10-ubo50, with $\reactive$ slightly lower for ubo100. The $\stnu$ method has the lowest feasibility rate, but the difference narrows as problem size increases. For $\epsilon=2$, the $\stnu$ achieves the highest feasibility ratios due to better handling of larger variances in durations.
\begin{table}[htb!]
    \fontsize{9}{11}\selectfont
    \setlength{\tabcolsep}{1mm}
    \centering
    \begin{tabular}{lrrrr|rrrr}  
        \toprule
        \multirow{2}{*}{Set} & \multicolumn{4}{c|}{$\epsilon = 1$} & \multicolumn{4}{c}{$\epsilon = 2$} \\
        \cmidrule(lr){2-5} \cmidrule(lr){6-9}
                             & $\stnu$ & $\saashort$ & $\pro_{0.9}$ & $\react$ & $\stnu$ & $\saashort$ & $\pro_{0.9}$ & $\react$ \\  
        \midrule
        j10                  &  0.65 &  0.85 &  0.85 &  0.85 &  0.63 &  0.63 &  0.64 &  0.63 \\
        j20                  &  0.65 &  0.76 &  0.76 &  0.76 &  0.63 &  0.54 &  0.53 &  0.53 \\
        j30                  &  0.78 &  0.89 &  0.89 &  0.89 &  0.63 &  0.51 &  0.48 &  0.49 \\
        ubo50                &  0.77 &  0.86 &  0.86 &  0.86 &  0.67 &  0.41 &  0.42 &  0.41 \\
        ubo100               &  0.84 &  0.91 &  0.91 &  0.88 &  0.79 &  0.35 &  0.36 &  0.33 \\
        \bottomrule
    \end{tabular}
      \caption{Feasibility Ratios for \(\epsilon = 1\) and \(\epsilon = 2\).} \label{tab:feasible_combined}
\end{table}

\subsubsection{Results Statistical Tests}

Table \ref{tab:test_results} shows a subset of the pairwise test results for the metrics solution quality, offline time, and online time. In our Appendix, we provide all test results per instance set / noise level $\epsilon$ setting. The outcomes of the test results can be used to make partial orderings of the methods, distinguishing between a strong partial ordering (Wilcoxon test) and a weak ordering (proportion test). Besides that, the results of the magnitude test give insight in the magnitude differences of each performance metric based on the double hits. For example, in Table \ref{tab:test_results}, the normalized average makespan is 0.958 for \stnu \ and 1.015 for \reactive, with a significant advantage for \stnu \ on double hits.

\subsubsection{Partial Ordering Visualization}
In Figures~\ref{fig:part_ord_quality_summary}--\ref{fig:part_ord_online_summary}, an arrow $A \rightarrow B$ indicates that in the majority of the settings either $A$ is consistently better than $B$ (Wilcoxon) and/or $A$ is better than $B$ a significant number of times (proportion). Due to space constraints and for clarity, we have chosen to display only the most common pattern per metric rather than all partial orderings per instance set and noise level. Consequently, the distinction between strong and weak partial orderings is omitted in these figures. We refer to the Technical Appendix for the partial orderings per setting including the distinction between a strong and weak partial ordering.

\subsection{Analysis}
 Figure~\ref{fig:part_ord_quality_summary} shows the visualization of the partial ordering for \textbf{solution quality} (makespan). The results are consistent for the different instance sets and noise levels. The $\stnu$
shows to be the outperforming method based on solution
quality. The $\reactive$ approach outperforms the
$\proactive$ methods. Furthermore, $\saa$ outperforms in many cases $\proactive_{0.9}$, although for larger
instances and a higher noise level a significant difference is not present. In earlier work, $\proactive$ approaches were considered state-of-the-art, but in our analysis, we found better makespan results for the STNU-based approach. We found that for each pair for which an arrow is visualized in Figure~\ref{fig:part_ord_quality_summary} also a significant magnitude difference was found on the double hits.

\begin{figure}[htb!]
    \centering
\scalebox{0.87}{
\begin{tikzpicture}
\begin{scope}[every node/.style={}]
    \node (A) at (0.5,0) {\stnu};
    \node (B) at (2.5,0) {\reactive};
    \node (C) at (5,0) {\saa};
    \node (D) at (8,0) {\proheur};
\end{scope}

\begin{scope}[>={Stealth[black]},
              every node/.style={},
              every edge/.style={draw=black}]
    \path [->] (A) edge node {} (B);
    \path [->] (A) edge[bend right=10] node {} (C);
    \path [->] (B) edge node {} (C);
    \path [->] (B) edge[bend left=10] node {} (D);
    \path [->] (C) edge node {} (D);
        \path [->] (A) edge[bend right=10] node {} (D);
\end{scope}
\end{tikzpicture}
}
    \caption{Summarizing illustration of the partial ordering of the different methods for solution quality.}
    \label{fig:part_ord_quality_summary}
\end{figure}

Figure~\ref{fig:part_ord_offline_summary} shows that the $\proactive_{0.9}$ and $\reactive$ have the lowest relative \textbf{offline runtime} and we found no significant difference between the two. The ordering of $\stnu$ and $\saa$ depends on the problem size (the ordering flips for larger instances). These relative orderings are confirmed with the magnitude test on double hits. However, we assign infinitely bad offline computation time to infeasible solutions. When executing the Wilcoxon test we include all instances for which at least one of the two methods generated a feasible solution. For that reason, a flip in the partial ordering occurs for $\epsilon=2$, ubo50 and ubo100: $\stnu$ shows the best performance, and $\proactive_{0.9}$ outperforms $\reactive$ according to the Wilcoxon tests. This was mainly due to the higher feasibility ratio for the better methods (see Table \ref{tab:feasible_combined}) as the results from the magnitude test on double hits contradicted Wilcoxon in these cases (we did not visualize this pattern in the main paper, but we included it in the Appendix). 

\begin{figure}[htb!]
    \centering
\scalebox{0.87}{
\begin{tikzpicture}
\begin{scope}[every node/.style={}]
    \node (stnu) at (8,0.5) {\stnu};
    \node (rea) at (0,0) {\reactive};
    \node (saa) at (3,0.5) {\saa};
    \node (pro) at (0,1) {\proheur};
\end{scope}

\begin{scope}[>={Stealth[black]},
              every node/.style={},
              every edge/.style={draw=black}]
    \path [->] (rea) edge node {} (saa);
    \path [->] (pro) edge node {} (saa);
    \path [->] (saa) edge[bend right=15] node[midway, above] {smaller instances} (stnu);
    \path [->] (stnu) edge[bend right=15] node[midway, below] {larger instances} (saa);
    \path [->] (rea) edge[bend right=12] node {} (stnu);
    \path [->] (pro) edge[bend left=12] node {} (stnu); 
\end{scope}
\end{tikzpicture}
}
    \caption{Summarizing illustration of the partial ordering of the different methods for time offline.}
    \label{fig:part_ord_offline_summary}
\end{figure}

We observe the general partial ordering for \textbf{online runtime} in Figure~\ref{fig:part_ord_online_summary}. The superiority of $\proactive_{0.9}$ and $\saa$ are expected, as these methods only require a feasibility check online. The faster online time of the $\stnu$ compared to the $\reactive$ can be explained by the fact that the $\stnu$ employs a polynomial real-time execution algorithm, while $\reactive$ calls a deterministic CP solver multiple times. These results are confirmed with a magnitude test on double hits. Again, there are a few settings ($\epsilon=2$, ubo50 and ubo100) in which the magnitude and the Wilcoxon test contradict each other: $\stnu$ outperforms the $\proactive$ methods based on the Wilcoxon test due to higher feasibility ratios, while the magnitude test on double hits shows better online runtime for $\proactive_{0.9}$ and $\saa$.

\begin{figure}[htb!]
    \centering
\scalebox{0.9}{
\begin{tikzpicture}
\begin{scope}[every node/.style={}]
    \node (rea) at (6,0.5) {\reactive};
    \node (saa) at (0,0) {\saa};
    \node (stnu) at (3,0.5) {\stnu};
    \node (pro) at (0,1) {\proheur};
\end{scope}

\begin{scope}[>={Stealth[black]},
              every node/.style={},
              every edge/.style={draw=black}]
    \path [->] (saa) edge node {} (stnu);
    \path [->] (pro) edge node {} (stnu);
   \path [->] (stnu) edge node {} (rea);
   \path [->] (pro) edge node {} (rea);
   \path [->] (saa) edge node {} (rea);
\end{scope}
\end{tikzpicture}
}
       \caption{Summarizing illustration of the partial ordering of the different methods for time online.}
    \label{fig:part_ord_online_summary}
\end{figure}

\section{Conclusion and Future Work}
This study introduces new scheduling methods for \SRCPSPmax and statistically benchmarks them, addressing the existing research gap of a lacking benchmarking paper.

Until now, proactive (\soruh) methods were considered the best method for \SRCPSPmax, although partial order schedules have shown potential in earlier research. We found that proactive methods can be improved with online rescheduling, resulting in better solution quality for the method $\reactive$ compared to proactive approaches $\saa$ and $\proactive_{\gamma}$. We find that the algorithm $\stnu$ that uses partial order schedules outperforms the other methods on solution quality in our evaluation. Although in general, $\proactive_{\gamma}$ and $\reactive$ have better offline computation time than $\stnu$, and $\saa$ and $\proactive_{\gamma}$ have better online computation time than $\stnu$, the $\stnu$ also showed good relative runtime results due to the polynomial time STNU-related algorithms. 

In future work, the same approach could be used to evaluate other scheduling problems, and we can gain more insight into how these methods perform on both well-known problems from the literature and in practical situations. Temporal constraints that are defined from end-to-start can be an interesting problem domain for future work. The $\stnu$ method can extended in multiple directions, e.g., rearranging the order of the jobs online or using probabilistic STNs are both interesting for future work.

Furthermore, the set of methods could even be broadened by including sequential approaches \cite{halperin2022reinforcement} or machine learning-based methods like the graph neural network in \cite{teichteil2023fast} for SRCPSP.

\subsubsection*{Reproducibility Statement}
 All our results can be reproduced following the README from our repository.

 \section*{Acknowledgements}
 This work is supported by the AI4b.io program, a collaboration between TU Delft and dsm-firmenich, and is fully funded by dsm-firmenich and the RVO (Rijksdienst voor Ondernemend Nederland).

\bibliographystyle{alpha}
\bibliography{sample}
\appendix
\newpage
\section*{Technical Appendix}

\section{Reproducibility}
We uploaded our code to make it possible to reproduce the results of our research and to facilitate others using a similar benchmarking method. Please refer to the \texttt{READ.ME} of our repository (\url{github.com/kimvandenhouten/AAAI25_SRCPSPmax}) to find instructions for requirements and directions for our experiment scripts and the scripts to run the statistical tests.

\section{Additional Explanation Stochastic Methods}

\subsection{SRCPSP/max Example}
 Let's reuse the example instance from the main paper (Figure~\ref{fig:example_graph_rcpsp_max}), but now we introduce uncertainty in the durations. We now suppose that activity $d$ follows a uniform discrete distribution with $P(d=1)=0.5$ and $P(d=2)=0.5$.

\subsection{STNU Method}

This section gives a more comprehensive explanation of the STNU method. We showcase what the temporal network for the given problem instance looks like and how we can see whether the network is DC.

\subsubsection{Step 1: Get fixed-point schedule by solving the deterministic RCPSP/max}

Suppose we use $\hat{d}=\{ 
 \hat{d}_a=2, \hat{d}_b=5, \hat{d}_c=3, \hat{d}_d=2, \hat{d}_e=2\}$, and obtain the following fixed-point solution:
\begin{figure}[htb]
    \centering
    \begin{subfigure}[t]{0.49\linewidth}
        \centering
        \includegraphics[width=\linewidth]{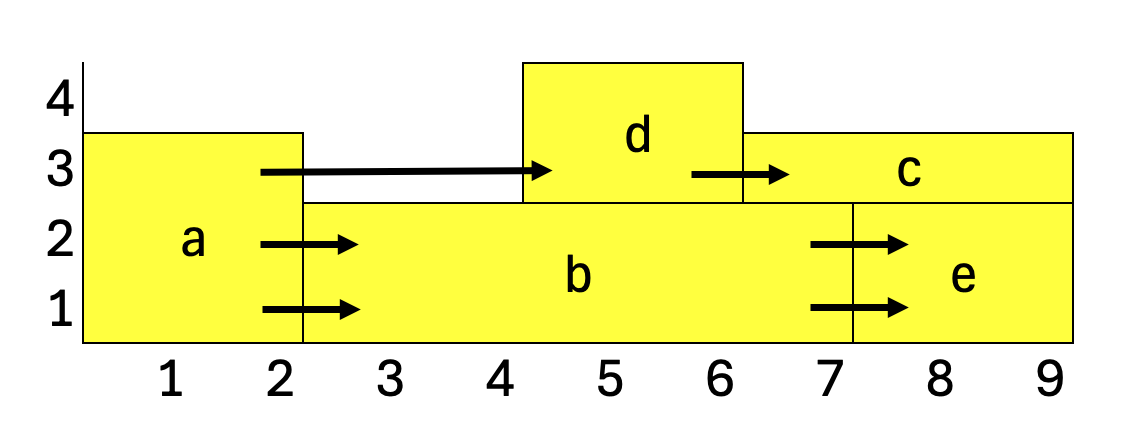}
        \caption{Fixed-point solution with resource chains.}
        \label{fig:resource_chains}
    \end{subfigure}
    \hfill
    \begin{subfigure}[t]{0.48\linewidth}
        \centering
        \includegraphics[width=\linewidth]{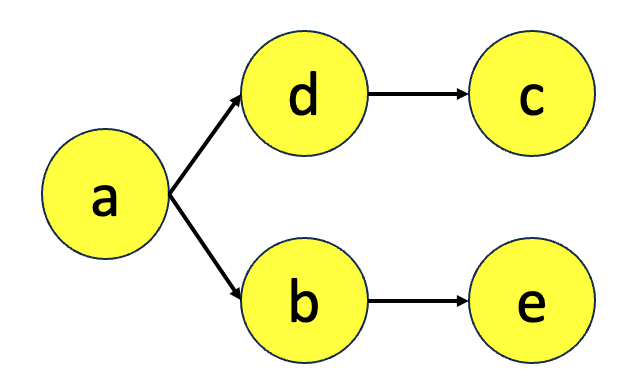}
        \caption{Partial order schedule based on resource chains.}
        \label{fig:pos_from_chains}
    \end{subfigure}
    \caption{Partial order schedules.}
    \label{fig:combined_figure}
\end{figure}

\subsubsection{Step 2: Construct STNU}
See in Figures \ref{fig:stnu_example_1} - \ref{fig:stnu_example_3} how an STNU is constructed step by step.
 \begin{figure}[h!]
    \centering
\begin{tikzpicture}
\begin{scope}[every node/.style={draw,circle}]
    \node (sa) at (-3,0) {start a};
    \node (fa) at (0,0) {finish a};
       \node (sb) at (3,-1.5) {start b};
    \node (fb) at (6,-1.5) {finish b};
      \node (sc) at (9,1.5) {start c};
    \node (fc) at (12,1.5) {finish c};
      \node (sd) at (3,1.5) {start d};
    \node (fd) at (6,1.5) {finish d};
         \node (se) at (9,-1.5) {start e};
    \node (fe) at (12,-1.5) {finish e};
\end{scope}

\begin{scope}[>={Stealth[black]},
              every node/.style={},
              every edge/.style={draw=black}]
    \path [->] (sa) edge[bend right=10] node[midway, below] {2} (fa);
    \path [->] (fa) edge[bend right=10] node[midway, above] {-2} (sa);
    
    \path [->] (sb) edge[bend right=10] node[midway, below] {5} (fb);
    \path [->] (fb) edge[bend right=10] node[midway, above] {-5} (sb);
    
        \path [->] (sc) edge[bend right=10] node[midway, below] {3} (fc);
    \path [->] (fc) edge[bend right=10] node[midway, above] {-3} (sc);
    
        \path [->] (sd) edge[bend right=10] node[midway, below] {d:1} (fd);
    \path [->] (fd) edge[bend right=10] node[midway, above] {D:-2} (sd);

          \path [->] (se) edge[bend right=10] node[midway, below] {2} (fe);
    \path [->] (fe) edge[bend right=10] node[midway, above] {-2} (se);

\end{scope}
\end{tikzpicture}
    \caption{Step 2a: For each activity, two nodes are created, representing its start and end, and Step 2b: We connect the start and end of each activity with a tight link (activities $a$, $b$, $c$, $e$) or a contingent link (activity $d$).}
    \label{fig:stnu_example_1}
\end{figure}

 \begin{figure}[h!]
    \centering
\begin{tikzpicture}
\begin{scope}[every node/.style={draw=gray,circle,text=gray}]
     \node (sa) at (-3,0) {start a};
    \node (fa) at (0,0) {finish a};
       \node (sb) at (3,-1.5) {start b};
    \node (fb) at (6,-1.5) {finish b};
      \node (sc) at (9,1.5) {start c};
    \node (fc) at (12,1.5) {finish c};
      \node (sd) at (3,1.5) {start d};
    \node (fd) at (6,1.5) {finish d};
         \node (se) at (9,-1.5) {start e};
    \node (fe) at (12,-1.5) {finish e};
\end{scope}

\begin{scope}[>={Stealth[gray]},
              every node/.style={text=gray},
              every edge/.style={draw=gray}]
    \path [->] (sa) edge[bend right=10] node[midway, below] {2} (fa);
    \path [->] (fa) edge[bend right=10] node[midway, above] {-2} (sa);
    
    \path [->] (sb) edge[bend right=10] node[midway, below] {5} (fb);
    \path [->] (fb) edge[bend right=10] node[midway, above] {-5} (sb);

          \path [->] (sc) edge[bend right=10] node[midway, below] {3} (fc);
    \path [->] (fc) edge[bend right=10] node[midway, above] {-3} (sc);
    
        \path [->] (sd) edge[bend right=10] node[midway, below] {d:1} (fd);
    \path [->] (fd) edge[bend right=10] node[midway, above] {D:-2} (sd);

          \path [->] (se) edge[bend right=10] node[midway, below] {2} (fe);
    \path [->] (fe) edge[bend right=10] node[midway, above] {-2} (se);
\end{scope}

    \begin{scope}[>={Stealth[black]},
              every node/.style={},
              every edge/.style={draw=black}]
    \path [->] (sd) edge[bend right=10] node[midway, left] {3} (se);
     \path [->] (se) edge[bend right=1] node[midway, left] {-3} (sd);
\path [->] (sb) edge[bend left=30] node[midway, above] {-2} (sa);
   \path [->] (sc) edge[bend right=18] node[midway, left] {-1} (sb);
      \path [->] (sa) edge[bend left=30] node[midway, below] {6} (sc);
\end{scope}
\end{tikzpicture}
    \caption{Step 2c: The minimal and maximal time lags are modeled. These edges correspond to the precedence graph of Figure~\ref{fig:example_graph_rcpsp_max}.}
    \label{fig:stnu_example_2}
\end{figure}

 \begin{figure}[h!]
    \centering
\begin{tikzpicture}
\begin{scope}[every node/.style={draw=gray,circle,text=gray}]
    \node (sa) at (-3,0) {start a};
    \node (fa) at (0,0) {finish a};
       \node (sb) at (3,-1.5) {start b};
    \node (fb) at (6,-1.5) {finish b};
      \node (sc) at (9,1.5) {start c};
    \node (fc) at (12,1.5) {finish c};
      \node (sd) at (3,1.5) {start d};
    \node (fd) at (6,1.5) {finish d};
         \node (se) at (9,-1.5) {start e};
    \node (fe) at (12,-1.5) {finish e};
\end{scope}

\begin{scope}[>={Stealth[gray]},
              every node/.style={text=gray},
              every edge/.style={draw=gray}]
    \path [->] (sa) edge[bend right=10] node[midway, below] {2} (fa);
    \path [->] (fa) edge[bend right=10] node[midway, above] {-2} (sa);
    
    \path [->] (sb) edge[bend right=10] node[midway, below] {5} (fb);
    \path [->] (fb) edge[bend right=10] node[midway, above] {-5} (sb);
    
        \path [->] (sc) edge[bend right=10] node[midway, below] {3} (fc);
    \path [->] (fc) edge[bend right=10] node[midway, above] {-3} (sc);
    
        \path [->] (sd) edge[bend right=10] node[midway, below] {d:1} (fd);
    \path [->] (fd) edge[bend right=10] node[midway, above] {D:-2} (sd);

          \path [->] (se) edge[bend right=10] node[midway, below] {2} (fe);
    \path [->] (fe) edge[bend right=10] node[midway, above] {-2} (se);

   \path [->] (sd) edge[bend right=10] node[midway, left] {3} (se);
     \path [->] (se) edge[bend right=1] node[midway, left] {-3} (sd);
\path [->] (sb) edge[bend left=30] node[midway, above] {-2} (sa);
   \path [->] (sc) edge[bend right=18] node[midway, left] {-1} (sb);
      \path [->] (sa) edge[bend left=30] node[midway, below] {6} (sc);
\end{scope}

    \begin{scope}[>={Stealth[black]},
              every node/.style={},
              every edge/.style={draw=black}]
   \path [->] (sd) edge[bend right=10] node[midway, above] {0} (fa);
  \path [->] (sc) edge[bend right=10] node[midway, above] {0} (fd);
    \path [->] (se) edge[bend right=10] node[midway, below] {0} (fb);
    \path [->] (sb) edge[bend right=10] node[midway, above] {0} (fa);
\end{scope}
\end{tikzpicture}
    \caption{Step 2d: The resource chains that construct the POS are added as additional edges as $(\Start_B, 0, \End_A)$ if activity A precedes activity B. Each arrow in Figure~\ref{fig:combined_figure} would lead to a resource chain edge.}
    \label{fig:stnu_example_3}
\end{figure}
\subsubsection{Step 3: Check DC}
The resulting network of Figure~\ref{fig:stnu_example_3} does not contain any negative cycles, and therefore the network is dynamically controllable \cite{morris2014dynamic}. 
Note that for this specific example, if we had used the fixed-point solution and resource chains from Figure~\ref{fig:chaining}, we would have found a negative cycle in the resulting STNU and, therefore, no dynamic controllability; see Figure~\ref{fig:non_dc}.

 \begin{figure}[h!]
    \centering
\begin{tikzpicture}
\begin{scope}[every node/.style={draw=black,circle,text=black}]
    \node (sd) at (-2,2) {start d};
    \node (fd) at (1,2) {finish d};
    
          \node (sa) at (1,0) {start a};
    \node (fa) at (4,0) {finish a};

    \node (sb) at (6,-2) {start b};
    \node (fb) at (9,-2) {finish b};
    
      \node (sc) at (9,2) {start c};
    \node (fc) at (12,2) {finish c};
    
 \node (se) at (4,2) {start e};
    \node (fe) at (7,2) {finish e};
\end{scope}

\begin{scope}[>={Stealth[black]},
              every node/.style={text=black},
              every edge/.style={draw=black}]
    \path [->] (sa) edge[bend right=10] node[midway, below] {2} (fa);
    \path [->] (fa) edge[bend right=10] node[midway, above] {-2} (sa);
    
    \path [->] (sb) edge[bend right=10] node[midway, below] {5} (fb);
    \path [->] (fb) edge[bend right=10] node[midway, above] {-5} (sb);
    
        \path [->] (sc) edge[bend right=10] node[midway, below] {3} (fc);
    \path [->] (fc) edge[bend right=10] node[midway, above] {-3} (sc);
    
        \path [->] (sd) edge[bend right=10] node[midway, below] {d:1} (fd);
    \path [->] (fd) edge[bend right=10] node[midway, above] {D:-2} (sd);

          \path [->] (se) edge[bend right=10] node[midway, below] {2} (fe);
    \path [->] (fe) edge[bend right=10] node[midway, above] {-2} (se);

    \path [->] (sd) edge[bend left=60] node[midway, above] {3} (se);
     \path [->] (se) edge[bend right=40] node[midway, above] {-3} (sd);
\path [->] (sb) edge[bend left=20] node[midway, below] {-2} (sa);
   \path [->] (sc) edge[bend right=10] node[midway, left] {-1} (sb);
         \path [->] (sa) edge[bend right=50] node[midway, above] {6} (sc);
\end{scope}

    \begin{scope}[>={Stealth[black]},
              every node/.style={},
              every edge/.style={draw=black}]
              
   \path [->] (sa) edge node[midway, left] {0} (fd);
  \path [->] (se) edge node[midway, right] {0} (fa);
    \path [->] (sc) edge[bend right=10] node[midway, below] {0} (fe);
    \path [->] (sb) edge[bend right=10] node[midway, above] {0} (fa);
\end{scope}

   \begin{scope}[>={Stealth[red]},
              every node/.style={},
              every edge/.style={draw=red}]
              
  \path [->] (sd) edge[loop, out=135, in=200, looseness=5] node[midway, left] {-1} (sd);
\end{scope}
\end{tikzpicture}
    \caption{This figure shows an example of a non-DC network (the resource chains from Figure~\ref{fig:chaining}), which is due to a negative cycle. The negative cycle can be detected by backward propagation (using the algorithm from \cite{morris2014dynamic}) and is the path: $\Start_D$ $\leftarrow$ $\End_D$ $\leftarrow$ $\Start_A$ $\leftarrow$ $\Start_E$ $\leftarrow$ $\End_A$ $\leftarrow$ $\Start_D$ with length $-1$.}
    \label{fig:non_dc}
\end{figure}

\subsubsection{Step 4: Run RTE*}
The RTE* \cite{posenato2022cstnu,hunsberger2024foundations} algorithm provides the online scheduling strategy that transforms the network into a schedule (i.e. it assigns execution times to each node). The RTE* adapts the execution times to the outcome of the contingent time points, resulting in schedule $\{s_a=0, s_b=2, s_c=5, s_d=4, s_e=7\}$ if $d_d=1$ and $\{s_a=0, s_b=2, s_c=6, s_d=4, s_e=7\}$ if $d_d=2$.
\newpage

\vspace{20pt}
\section{Tuning}\label{app:tuning}
This section describes the tuning process for the different scheduling methods. First, we investigate the effect of the time limit on solving the deterministic CP. Then, for the proactive approach, we tuned the time limit and the sample selection for the Sample Average Approximation (SAA). For $\reactive$, we tuned the quantile approximation and the time limit for rescheduling. 
\subsection{Comparing CP and MIP for RCPSP/max}

This section shows a comparison of a CP-based formulation and MILP-based formulation of the RCPSP/max problem. The MILP model that we use in the comparison is the following: 

\subsubsection*{MILP Formulation for RCPSP/max}

\textbf{Sets:}
\begin{align*}
    J & : \text{ set of tasks, indexed by } j, \\
    R & : \text{ set of resources, indexed by } r, \\
    T & : \text{ set of discrete time points, indexed by } t.
\end{align*}

\textbf{Parameters:}
\begin{align*}
    p_j & : \text{duration of task } j, \quad \forall j \in J, \\
    b_r & : \text{capacity of resource } r, \quad \forall r \in R, \\
    l_{r,j} & : \text{demand of task } j \text{ on resource } r, \quad \forall r \in R, j \in J, \\
    \text{TemporalConstraints} & : \text{ set of precedence relationships } (\text{pred}, \text{lag}, \text{suc}).
\end{align*}

\textbf{Decision Variables:}
\begin{align*}
    x_{j,t} & : \begin{cases} 
    1 & \text{if task } j \text{ starts at time } t, \\
    0 & \text{otherwise}, 
    \end{cases} \quad \forall j \in J, t \in T, \\
    Makespan &: \text{makespan (objective to minimize)}.
\end{align*}

The MIP model is given by:
\begin{subequations}
\begin{align*}
 \text{Min } \Makespan   \quad  \text{s.t} \\
    \sum_{t \in T} t \cdot x_{j,t} + p_j,  \leq Makespan;  & \  \forall j \in J
 \\
    \sum_{t \in T} t \cdot x_{i,t} + \text{lag} \leq \sum_{t \in T} t \cdot x_{j,t}; & \ \forall (i, \text{lag}, j) \in \text{TemporalConstraints}
    \\
    \sum_{j \in J} \sum_{\tau = t - p_j + 1}^{t} l_{r,j} \cdot x_{j,\tau} \leq b_r; & \ \forall r \in R, \, \forall t \in T
 \\
    \sum_{t \in T} t \cdot x_{i,t} \geq 0; & \ \forall j \in J
\\
    \sum_{t \in T} x_{j,t} = 1; & \ \forall j \in J
\\
\end{align*}
\end{subequations}

\subsubsection*{Comparison MILP and CP}

\begin{table}[!htb]
    \centering
    \begin{tabular}{l|l|l|l}
        Instance folder & \# Instances & Average time CP & Average of time MIP \\ \hline
        j10 & 50 & 0,02 & 0,41 \\ 
        j20 & 50 & 24,34 & 17,04 \\ 
        j30 & 50 & 1,28 & 68,58 \\ 
        ubo50 & 50 & 75,09 & 210,96 \\ 
        ubo100 & 50 &135,30	& 510,75 \\ \
    \end{tabular}
    \caption{Runtime CP vs MILP, timelimit set to 600 seconds.}
\end{table}

\begin{table}[!htb]
    \centering
    \begin{tabular}{c|c|cc}
        Instance folder & \# Instances & \# Count CP & \# Count MIP  \\ \hline
        j10 & 50 & 0 & 0 \\ 
        j20 & 50& 2 & 0 \\ 
        j30& 50  & 0 & 5 \\ 
        ubo50 & 50 & 6 & 14 \\ 
        ubo100 & 50 & 11 & 39 \\ 
    \end{tabular}
    \caption{Count hit timelimit (10 minutes) CP vs MILP}
\end{table}

\subsection{Time Limit for CP on Deterministic Instances}
To understand how the time limit may affect the results, we first consider the deterministic instances of RCPSP/max.
\cite{fu2016proactive} reported on their time budget that \textit{ "$\soru$ was able to obtain solutions within 5 minutes for every one instance in J10 and 2 hrs for the J20 instances. However, for J30, we were unable to get optimal solutions for certain instances in the cut-off limit of 3 hrs. On the other hand, $\soruh$ was able to generate solutions for J10 instances within half of a second, J20 instances within 10 seconds, and J30 instances within 10 minutes on average."}

In our research, we extend the instance sets with ubo50 and ubo100, for which the time limit can become more crucial. For each set (j10-ubo100), we fixed the first 50 instances from the PSPlib \cite{Kolisch1996PSPLIBA} (PSP1 - PSP50). For j10-j30, the IBM CP Optimizer could solve all instances (PSP1-PSP50)  within 60 seconds to optimality or prove infeasibility. 

For the ubo50 instances PSP1 to PSP10 and the ubo100 instances PSP1 to PSP10, the effect of the time limits of 60 seconds, 600 seconds, and 3600 seconds respectively, is presented in Table~\ref{tab:time_limit} and~\ref{tab:my_label}. We observe that the solver status does not change when increasing the time limit to 600 seconds and only flips 1 instance from feasible to infeasible when increasing the time limit from 600 seconds to 3600 seconds. The makespan only improves significantly for the ubo100 instances for higher time limits. In the remaining experiments, we fixed the time limit to 60 seconds for solving deterministic CPs.
 
\begin{table}[htb!]
\centering
        \scalebox{1}{
    \begin{tabular}{c|c|c|c|c|c}
      Instance set & Time limit & Feasible & Infeasible &Optimal & Unkown  \\ \hline
       ubo50   & 60s & 3 & 3 & 2& 2\\ 
      ubo50   & 600s &3 & 3& 2& 2\\ 
        ubo50   & 3600s & 2&3 & 3& 2\\ \hline
       ubo100   & 60s &4 & 5& 0& 1\\ 
      ubo100   & 600s & 4& 5& 0& 1\\ 
        ubo100   & 3600s & 4& 5& 0&1 \\
    \end{tabular}}
    \caption{Count per solver status IBM CP solver after different time limits on deterministic instances.}
    \label{tab:time_limit}
\end{table}

\begin{table}[htb!]
    \centering
    \begin{tabular}{c|c|c|c}
       &  60s & 600s & 3600s   \\ \hline
      ubo50  &  204 & 200  & 200 \\
       ubo100  & 421 & 410 & 410 \\
    \end{tabular}
    \caption{Average makespan (filtered feasible and optimal solutions on ubo50 and ubo100 set.}
    \label{tab:my_label}
\end{table}

\subsection{Proactive Approach}
The sample selection process is expected to influence the performance of the proactive method. We used a subset of the j10 and ubo50 instances (for both sets $PSP\_1$-$PSP\_20$, and used 10 duration samples per instance. We fixed the time limit to $60$ seconds and compared the feasibility ratios in Table \ref{tab:proactive_quantile}. We found that the higher feasibility ratios were obtained using two settings for the proactive method, being 1 sample with $\gamma=0.9$, to which we will refer with $\proactive_{0.9}$, and the setting with the smart samples to which we will refer with $\saa$ in the remaining of our experiments. For the SAA with four samples (smart samples), we investigated the effect of the time limit on the makespan in Table \ref{tab:saa_time_limit}. We observed that the makespan would still improve while making the step from 600 seconds to 3600 seconds. We decide to use a time limit of 1800 seconds instead of 3600 seconds in the remaining of the experiments to be able to conduct more experiments.

\begin{table}[!htb]
    \centering
        \scalebox{1}{
    \begin{tabular}{l|llll}
     1 sample   & $\gamma=$0.5 & $\gamma=$0.75 & $\gamma=$0.9 & $\gamma=$1 \\ \hline
       j10 &   0.11  & 0.53 & 0.94 & 0.75  \\ 
       ubo50 & 0  & 0.01  & 0.92 & 0.80   \\ \hline
     multiple samples       & n=5 & n=10 & n=25 & n=50  \\ \hline
    j10 & 0.63	& 0.78	& 0.94	& 0.94 \\
   ubo50 &   0.12	& 0.48 &	0.89	 & 0.91 \\ \hline
        smart samples*   & n=4  &    \\ \hline
    j10 & 0.94	&  \\
   ubo50 &   0.92	
    \end{tabular}}
   \caption{Feasibility ratio for different $\gamma$-quantiles in proactive approach. Each cell counts hits out of 200 experiments. * Uses a combination of smart samples that are the quantiles with $\gamma \in [0.25, 0.5, 0.75, 0.9]$} \label{tab:proactive_quantile}
\end{table}

\begin{table}[htb!]
    \centering
    \begin{tabular}{c|c|c|c}
       &  60s & 600s & 3600s   \\ \hline
      ubo50  &  233 & 232  & 231 \\
       ubo100  & 485 & 480 & 472 \\
    \end{tabular}
    \caption{Average makespan obtained with $\saa$ with four scenarios for different time limits (filtered feasible and optimal solutions on ubo50 and ubo100 set.}
    \label{tab:saa_time_limit}
\end{table}

\subsection{Reactive Approach}
First, we observed the effect of the duration estimations on the performance of $\reactive$. We used a subset of the j10 and ubo50 instances (for both sets $PSP\_1$-$PSP\_20$, and used 10 duration samples per instance. We fixed the time limit for the initial schedule to $60$ seconds and for the rescheduling to $2$ seconds.

\begin{table}[!htb]
    \centering
        \scalebox{1}{
    \begin{tabular}{l|llll}
       & $\gamma=$0.5 & $\gamma=$0.75 & $\gamma=$0.9 & $\gamma=$1 \\ \hline
       j10 &   0.11  & 0.53 & 0.94 & 0.75  \\ 
       ubo50 & 0  & 0.01  & 0.92 & 0.80   \\ \hline
    \end{tabular}}
   \caption{Feasibility ratio for different $\gamma$-quantiles in reactive approach. Each cell counts hits out of 200 experiments.} \label{tab:reactive_quantile}
\end{table}

Remarkably, we observe similar feasibility ratios to the proactive approach, indicating that for feasibility the initial schedule is quite important. For the final evaluation, we fixed $\gamma=0.9$ for $\reactive$, and will analyze how the solution quality improves with the rescheduling procedure compared to a standard proactive approach. 

Next, we observe the effect of the time limit for rescheduling.  We used the same subset of the j10 and ubo50 instances (for both sets $PSP\_1$-$PSP\_20$, used 10 duration samples per instance, and runtime limits of 1, 2, 10 and 30 seconds. The results (in Table~\ref{tab:reactive_quantile}) are almost similar for the different time limits, and this might be because the solver finishes already before the time limit, and the increase in time online has mainly to do with the number of solver calls. In the experimental evaluation, we therefore fixed the rescheduling time limit to $2$ seconds, which we expected to be sufficient for larger or slightly more complicated problems. 

\begin{table}[!htb]
    \centering
        \scalebox{1}{
    \begin{tabular}{l|l|llll}
     &  Time limit rescheduling & 1s &  2s & 10s & 30s \\ \hline
       j10  & makespan & 38  & 38  & 38   & 39\\
       j10  & time online & 0.03 & 0.04  & 0.04  & 0.04  \\ 
          ubo50 & makespan  & 171   & 171 & 171 & 172 \\
       ubo 50 & time online & 15.4 & 15.1  & 15.1  & 15.4  \\ 
    \end{tabular}}
   \caption{Average makespan and time online for double hits on different time limits for rescheduling. Each cell averages over double hits out of 200 experiments.} \label{tab:reactive_timelimit}
\end{table}

\subsection{Hyperparameters Selection}

This subsection presents the hyperparameters in Table \ref{tab:hyperparameters} that are used in the final experiments that are also presented in the main paper. 
\begin{table}[htb!]
    \centering
     \scalebox{1}{
    \begin{tabular}{l|l|l|l|l} 
       Hyperparameter  &  $\saa$ & $\proactive_{0.9}$ & $\stnu$ & $\reactive$ \\ \hline
        $\gamma$ & [0.25, 0.5, 0.75, 0.7]  & 0.9 & 1  & 0.9  \\
        Time limit CP & 1800s & 600s & 600s & 600s and 2s \\
        Solver & IBM CP  & IBM CP & IBM CP & IBM CP
    \end{tabular}}
    \caption{Hyperparameter settings.}
    \label{tab:hyperparameters}
\end{table}

\section{Statistical Tests}\label{app:statistical}
\subsection{Wilcoxon Test}
The Wilcoxon test that we use follows \cite{cureton1967normal} and is described below: 

\begin{itemize}
    \item Collect a set of matched pairs (the results from two different methods on one instance sample). 
    \item Compute the difference between the two test results for each pair.
    \subitem An important remark is that because of the discrete objective values (makespan), we can obtain a zero difference when there is a tie, we use Pratt's procedure for handling ties \cite{pratt1959remarks}, which includes zero-differences in the ranking process, but drops the ranks of the zeros afterward.
    \item Order the pairs according to the absolute values of the differences.
    \item Assign ranks to the pairs based on these absolute values.
    \item Sum the positive ranks ($T_{pos})$ and the negative ($T_{neg})$ ranks separately. 
    \item Take the smaller of the two $T=min(T_{pos}, T_{neg})$.
    \item If the two methods have no consistent difference in their relative performances, then the rank sums should be approximately equal. This is tested with a normal approximation for the Wilcoxon statistic, which is outlined by \cite{cureton1967normal}. \cite{cureton1967normal} propose a corrected normal approximation, which is needed because of the usage of the Pratt procedure for handling zero differences. 
     \item The normalized Z-statistic is given by the formula: $z = (T - d) / se$, where $d$ is the continuity correction from \cite{cureton1967normal}, and $se$ is the standard error. 
\end{itemize}
All of the above can be executed using the Python package SciPy \cite{2020SciPy-NMeth} built-in method $scipy.stats.wilcoxon$ that is called with parameters $method$="approx", "$zero\_method$="pratt", and $correction$=True.

\subsection{Z-test for Proportion (Binomial Distribution)}
We use \cite{100tests} as a reference for the Z-test for Proportion. It is important to mention that this test is approximate as it assumes that the number of observations justifies a normal approximation for the binomial. (In contradiction to the SciPy package and its built-in method $scipy.stats.binomtest$ containing an exact test). 

The proportion test investigates whether the is a significant difference between the assumed proportion of wins $p_0$ and the observed proportion of wins $p$. In our analysis, two methods are compared, and the number of wins for each method is counted based on one metric. 

The procedure for the proportion test is as follows:
\begin{itemize}
    \item Collect a set of matched pairs (the results from two different methods on one atomic instance form a pair). 
    \item For each pair, determine which method wins, and count the wins for both methods. Exclude all ties.
    \item Calculate the ratio of wins for one of the two methods.
    \item Test where this ratio differs significantly from $p_0=0.5$ (equal probability of winning).
    \begin{itemize}
        \item The test statistic is $Z=\frac{|p-p_0|-\frac{1}{2n}}{\{ \frac{p_0(1-p_0)}{n}^{1/2} \}}$
        \item The term $\frac{1}{2n}$ in the numerator is a discontinuity correction.
        \item For a two-sided test with a significance level $\alpha=0.05$ the acceptance region for the null hypothesis is $-1.96 < Z < 1.96$.
    \end{itemize}
\end{itemize}.

\subsection{Magnitude Test}
The magnitude test we use is a t-test for two population means, or the method of paired comparisons such as Test 10 in the book by \cite{100tests}. The test is whether there is a significant difference between the two population means. The procedure for this paired comparison t-test is as follows:
\begin{itemize}
    \item Collect a set of matched pairs (so the results from two different methods on one atomic instance form a pair). 
    \item Normalise the performances for each pair by computing the mean value of the pair and dividing the two items in the pair by the pairs' mean such that all normalized observations will be between 0 and 2, and 1 indicates a tie.
    \item Compute the differences $d_1$ between the two test results for each pair $i$.
    \item Compute the variances of the differences with the following formula: $s^2=\sum_{i=1}^n \frac{(d_i-\bar{d})^2}{n-1}.$
    \item Compute the means of both methods $\bar{x_1}$ and $\bar{x_2}$.  
    \item Compute the t-statistic using the formula: $t=\frac{\bar{x_1}-\bar{x_2}}{s/n^{1/2}}$
    \item Test significance by checking whether $t$ lies within the acceptance region for which the values are given by the Student's t-distribution (two-sided) with $n-1$ degrees of freedom.
\end{itemize}

After the normalization step, it is possible to execute the test with the Scipy package and, specifically, its built-in method $scipy.stats.ttest\_rel$.

\section{Results}
Please find the results of the statistical test in this section. In Section \ref{app:feasibility}, we discuss the feasibility rations. In Section \ref{app:partial_orderings}, we present the visualizations of the partial orderings. In Section \ref{app:magnitude}, we present the results of the magnitude tests on the double hits that can give some insight into the magnitude differences of the performance metrics between the different methods.

The test results from the Wilcoxon test and the proportion tests are used to make figures for the partial orderings, and the complete tables with the results can be found in the following tables at the end of this document. 
\begin{itemize}
    \item Table \ref{tab:obj_pairwise_1} for noise level $\epsilon=1$ and Table \ref{tab:obj_pairwise_2} for noise level $\epsilon=2$: including the results of the Wilcoxon and proportion tests on solution quality.
    \item Table \ref{tab:offline_pairwise_1} for noise level $\epsilon=1$ and Table \ref{tab:offline_pairwise_2} for noise level $\epsilon=2$: including the results of the Wilcoxon and proportion tests on time offline.
    \item Table \ref{tab:online_pairwise_1} for noise level $\epsilon=1$ and Table \ref{tab:online_pairwise_2} for noise level $\epsilon=2$: including the results of the Wilcoxon and proportion tests on online computation time.
\end{itemize}

\subsection{Feasibility Ratio}\label{app:feasibility}
First, we analyze the feasibility ratio obtained by the different methods. For $\epsilon=1$, we observe the feasibility ratio obtained by $\saa$, $\proactive_{0.9}$, $\reactive$ are the same for instance sets j10 - ubo50, and only for ubo100 $\reactive$ is a little bit lower than the $\saa$, $\proactive_{0.9}$. The $\stnu$ method has the lowest feasibility rate, but remarkably, the difference becomes smaller when the size of the problem grows (the difference is the smallest for ubo100).  

We observe a different pattern for $\epsilon=2$, with a higher noise factor. The highest feasibility ratios are obtained by $\stnu$. This could be explained by the fact that the $\stnu$ method uses the information about the distribution and is, therefore, better at handling the larger variances in activity duration.

\begin{table}[htb!]
    \centering
\begin{tabular}{l|rrrr}
 Set &  $\stnu$ &  $\saa$ &  $\proactive_{0.9}$ &  $\reactive$ \\ \hline
j10             &  0.65 &           0.85 &           0.85 &      0.85 \\
j20             &  0.65 &           0.76 &           0.76 &      0.76 \\
j30             &  0.78 &           0.89 &           0.89 &      0.89 \\
ubo50           &  0.77 &           0.86 &           0.86 &      0.86 \\
ubo100          &  0.84 &           0.91 &           0.91 &      0.88 \\
\end{tabular}
 \caption{Feasibility ratio for noise factor $\epsilon=1$ We use abbreviations for the proactive approach with $\gamma=0.9$ ($\pro_{0.9}$), and the SAA variant ($\saashort$).}
     \label{tab:feasible_1}
\end{table}

\begin{table}[htb!]
    \centering
\begin{tabular}{l|rrrr}
 Set &  $\stnu$ &  $\saa$ &  $\proactive_{0.9}$ &  $\reactive$ \\ \hline
j10             &  0.63 &           0.63 &           0.64 &      0.63 \\
j20             &  0.63 &           0.54 &           0.53 &      0.53 \\
j30             &  0.63 &           0.51 &           0.48 &      0.49 \\
ubo50           &  0.67 &           0.41 &           0.42 &      0.41 \\
ubo100          &  0.79 &           0.35 &           0.36 &      0.33 \\
\end{tabular}
 \caption{Feasibility ratio for noise factor $\epsilon=2$ We use abbreviations for the proactive approach with $\gamma=0.9$ ($\pro_{0.9}$), and the SAA variant ($\saashort$).}
     \label{tab:feasible_2}
\end{table}

\subsection{Partial Orderings}\label{app:partial_orderings}
\subsubsection{Partial Ordering Solution Quality}

Next, we analyze the test results based on the Wilcoxon test, proportion test, and magnitude test for solution quality. We translate the test results into partial orderings that are visualized with solid arrows indicating a significant difference based on the Wilcoxon test and dashed arrows indicating a significant difference based on the proportion test for the pairs where the Wilcoxon test did not find significant results. Furthermore, for all pairs where we find a significant difference according to Wilcoxon, or the proportion test, we analyze whether we can find a significant magnitude difference on the double hits. 
\begin{itemize}
\item First, we discuss partial orderings for the noise level $\epsilon=1$, which are visualized in Figure~\ref{fig:partial_ordering_quality_1}. For $\epsilon=1$, the partial ordering shows the same pattern across the different instance sets. However, there are a few exceptions: 
\begin{itemize}
    \item In some cases we can only show a dashed arrow, indicating a weaker partial ordering, that is a significant proportional difference (proportion of wins) instead of a significant difference obtained by Wilcoxon. 
    \item Only for ubo50, all arrows are solid, indicating that for all pairs, we found a significant difference according to the Wilcoxon test. 
    \item For ubo100, there is no significant difference between $\saa$ and $\proactive_{0.9}$. Possibly, for the larger instances, the SAA becomes more difficult to solve, for which reason it is not better than the heuristic $\gamma$-quantile procedure anymore.
\end{itemize}
\item Then, we analyze the partial orderings for the noise level $\epsilon=2$, in Figure~\ref{fig:partial_ordering_quality_2}. Here there are a few things to observe:
\begin{itemize} 
    \item We observe a shift in the partial orderings for $\epsilon=2$.
    \item For j10-j20, the partial ordering is similar to $\epsilon=1$, although for $\epsilon=2$, all arrows are solid (indicating a relation obtained by the Wilcoxon test).
    \item For $j30$, there is no significant difference between $\proactive_{0.9}$ and $\saa$, neither with Wilcoxon and neither proportionally. 
    \item The partial orderings for $j30$ and $ubo50$ are similar. 
    \item For ubo100, we observe the $\stnu$ is significantly better than the three other methods, although no significant difference can be found between the three other methods.
\end{itemize}
\item The main conclusion is that in all situations the $\stnu$ shows to be the outperforming method based on solution quality. In general, the $\reactive$ shows to outperform the \proactive methods. Furthermore, $\saa$ outperforms in many cases $\proactive_{0.9}$, although for larger instances and a higher noise level this difference is not present anymore.
\end{itemize} 

\begin{figure}[htb!]
\centering
    \begin{tikzpicture}
\begin{scope}[every node/.style={}]
    \node (A) at (0,0) {\stnu};
    \node (B) at (2,0) {\reactive};
    \node (C) at (5,0) {\saa};
    \node (D) at (8,0) {\proheur};
   \node (instance) at (12,0) {j10};
\end{scope}

\begin{scope}[>={Stealth[black]},
              every node/.style={},
              every edge/.style={draw=black}]
    \path [->,dashed] (A) edge node {} (B);
    \path [->,dashed] (A) edge[bend right=10] node {} (C);
    \path [->] (B) edge node {} (C);
    \path [->] (B) edge[bend left=10] node {} (D);
    \path [->] (C) edge node {} (D);
        \path [->,dashed] (A) edge[bend right=10] node {} (D);
\end{scope}
\end{tikzpicture}

    \begin{tikzpicture}
\begin{scope}[every node/.style={}]
    \node (A) at (0,0) {\stnu};
    \node (B) at (2,0) {\reactive};
    \node (C) at (5,0) {\saa};
    \node (D) at (8,0) {\proheur};
     \node (instance) at (12,0) {j20};
\end{scope}

\begin{scope}[>={Stealth[black]},
              every node/.style={},
              every edge/.style={draw=black}]
    \path [->,dashed] (A) edge node {} (B);
    \path [->] (A) edge[bend right=10] node {} (C);
    \path [->] (B) edge node {} (C);
    \path [->] (B) edge[bend left=10] node {} (D);
    \path [->] (C) edge node {} (D);
        \path [->] (A) edge[bend right=10] node {} (D);
\end{scope}
\end{tikzpicture}

    \begin{tikzpicture}
\begin{scope}[every node/.style={}]
    \node (A) at (0,0) {\stnu};
    \node (B) at (2,0) {\reactive};
    \node (C) at (5,0) {\saa};
    \node (D) at (8,0) {\proheur};
    \node (instance) at (12,0) {j30};
\end{scope}

\begin{scope}[>={Stealth[black]},
              every node/.style={},
              every edge/.style={draw=black}]
    \path [->, dashed] (A) edge node {} (B);
    \path [->, dashed] (A) edge[bend right=10] node {} (C);
    \path [->] (B) edge node {} (C);
    \path [->] (B) edge[bend left=10] node {} (D);
    \path [->] (C) edge node {} (D);
        \path [->] (A) edge[bend right=10] node {} (D);
\end{scope}
\end{tikzpicture}

    \begin{tikzpicture}
\begin{scope}[every node/.style={}]
    \node (A) at (0,0) {\stnu};
    \node (B) at (2,0) {\reactive};
    \node (C) at (5,0) {\saa};
    \node (D) at (8,0) {\proheur};
   \node (instance) at (12,0) {ubo50};
\end{scope}

\begin{scope}[>={Stealth[black]},
              every node/.style={},
              every edge/.style={draw=black}]
    \path [->] (A) edge node {} (B);
    \path [->] (A) edge[bend right=10] node {} (C);
    \path [->] (B) edge node {} (C);
    \path [->] (B) edge[bend left=10] node {} (D);
    \path [->] (C) edge node {} (D);
        \path [->] (A) edge[bend right=10] node {} (D);
\end{scope}
\end{tikzpicture}

    \begin{tikzpicture}
\begin{scope}[every node/.style={}]
    \node (A) at (0,0) {\stnu};
    \node (B) at (2,0) {\reactive};
    \node (C) at (5,0) {\saa};
    \node (D) at (8,0) {\proheur};
     \node (instance) at (12,0) {ubo100};
\end{scope}

\begin{scope}[>={Stealth[black]},
              every node/.style={},
              every edge/.style={draw=black}]
    \path [->] (A) edge node {} (B);
    \path [->] (B) edge node {} (C);
\end{scope}
\end{tikzpicture}
    \caption{Partial ordering on solution quality for $\epsilon=1$. The solid arrow indicates a significant difference obtained by the Wilcoxon test, and the dashed arrow shows only a proportional difference.}
    \label{fig:partial_ordering_quality_1}
\end{figure}

\begin{figure}[htb!]
\centering
    \begin{tikzpicture}
\begin{scope}[every node/.style={}]
    \node (A) at (0,0) {\stnu};
    \node (B) at (2,0) {\reactive};
    \node (C) at (5,0) {\saa};
    \node (D) at (8,0) {\proheur};
   \node (instance) at (12,0) {j10};
\end{scope}

\begin{scope}[>={Stealth[black]},
              every node/.style={},
              every edge/.style={draw=black}]
    \path [->] (A) edge node {} (B);
    \path [->] (A) edge[bend right=10] node {} (C);
    \path [->] (B) edge node {} (C);
    \path [->] (B) edge[bend left=10] node {} (D);
    \path [->] (C) edge node {} (D);
        \path [->] (A) edge[bend right=10] node {} (D);
\end{scope}
\end{tikzpicture}

    \begin{tikzpicture}
\begin{scope}[every node/.style={}]
    \node (A) at (0,0) {\stnu};
    \node (B) at (2,0) {\reactive};
    \node (C) at (5,0) {\saa};
    \node (D) at (8,0) {\proheur};
     \node (instance) at (12,0) {j20};
\end{scope}

\begin{scope}[>={Stealth[black]},
              every node/.style={},
              every edge/.style={draw=black}]
    \path [->] (A) edge node {} (B);
    \path [->] (A) edge[bend right=10] node {} (C);
    \path [->] (B) edge node {} (C);
    \path [->] (B) edge[bend left=10] node {} (D);
    \path [->] (C) edge node {} (D);
        \path [->] (A) edge[bend right=10] node {} (D);
\end{scope}
\end{tikzpicture}

    \begin{tikzpicture}
\begin{scope}[every node/.style={}]
    \node (A) at (0,0) {\stnu};
    \node (B) at (2,0) {\reactive};
    \node (C) at (5,0-0.5) {\saa};
    \node (D) at (5,0.5) {\proheur};
    \node (instance) at (12,0) {j30};
\end{scope}

\begin{scope}[>={Stealth[black]},
              every node/.style={},
              every edge/.style={draw=black}]
    \path [->] (A) edge node {} (B);
    \path [->] (A) edge node {} (C);
    \path [->] (B) edge node {} (C);
    \path [->] (B) edge node {} (D);
        \path [->] (A) edge node {} (D);
\end{scope}
\end{tikzpicture}

    \begin{tikzpicture}
\begin{scope}[every node/.style={}]
    \node (A) at (0,0) {\stnu};
    \node (B) at (2,0) {\reactive};
    \node (C) at (5,0-0.5) {\saa};
    \node (D) at (5,0.5) {\proheur};
    \node (instance) at (12,0) {ubo50};
\end{scope}

\begin{scope}[>={Stealth[black]},
              every node/.style={},
              every edge/.style={draw=black}]
    \path [->] (A) edge node {} (B);
    \path [->] (A) edge node {} (C);
    \path [->] (B) edge node {} (C);
    \path [->] (B) edge node {} (D);
        \path [->] (A) edge node {} (D);
\end{scope}
\end{tikzpicture}

    \begin{tikzpicture}
\begin{scope}[every node/.style={}]
    \node (A) at (0,0) {\stnu};
    \node (B) at (2.8,-0.5) {\reactive};
    \node (C) at (3.1,0) {\saa};
    \node (D) at (3,0.5) {\proheur};
     \node (instance) at (12,0) {ubo100};
\end{scope}

\begin{scope}[>={Stealth[black]},
              every node/.style={},
              every edge/.style={draw=black}]
    \path [->] (A) edge node {} (B);
    \path [->] (A) edge node {} (C);
     \path [->] (A) edge node {} (D);
\end{scope}
\end{tikzpicture}
    \caption{Partial ordering on solution quality for $\epsilon=2$. Solid arrow indicating a significant difference obtained by the Wilcoxon test, and dashed arrow only a proportional difference.}
    \label{fig:partial_ordering_quality_2}
\end{figure}

\subsubsection{Partial Ordering Time Offline}

\begin{itemize}
\item We expected the $\proactive_{0.9}$ and $\reactive$ to be the fastest offline, followed by the $\saa$ and $\stnu$, where the partial ordering between the latter two depends on the problem size. When two methods have (almost) only ties (such as $\proactive_{0.9}$ and $\reactive$), not all tests can be executed, and a $nan$ will appear. 
\item Importantly, the tests are designed in such a way that infeasibilities are also weighted in significance testing. When a method results in an infeasible solution, infinite offline time is assigned to this experiment. For that reason, it can still occur that we find a difference between $\reactive$ and $\proactive_{0.9}$: although they have the same offline procedure, one of the two can still fail, which results in infinite offline time.
\item In general, we give preference to the test results of the Wilcoxon test because this test shows a stronger relation than the proportional test, but to fill in the missing information, we use the proportion test to test if there is a significant difference in the proportion of wins. We noted, however, that the test outcomes of the two tests can sometimes be contradicting. The Wilcoxon test penalizes infeasibilities more severely, while the proportion test only considers the number of wins. Consequently, it is possible that method 1 produces more infeasible solutions but achieves better metrics, whereas method 2 generates fewer infeasible solutions but worse metrics, resulting in the Wilcoxon test on instances where at least one feasible solution exists tells us method 2 performs better (infeasible solutions are penalized more heavily). Proportion test on instances where at least one feasible solution exists tells us method 1 performs better. The magnitude test on double hits tells us Method 1 performs better as it has better metric values.
\item Now, we observe the different partial orderings obtained for $\epsilon=1$:
\begin{itemize}
  \item In general, the time spent offline for the $\proactive_{0.9}$ and the $\reactive$ is exactly the same, therefore in most cases we also did not find a significant difference between the two. However, especially for the ubo100 instances, we found that sometimes $\proactive_{0.9}$ had more feasible solutions, resulting in a partial ordering of $\proactive_{0.9} \rightarrow \reactive$.
  \item We observe that for j10 and j20 $\saa$ is consistently faster offline than $\stnu$, but this flips for j30 - ubo100, where the $\stnu$ becomes faster. This can be explained by the fact that the solve time of the $SAA$ can grow exponentially, while the DC-checking algorithm is a polynomial time algorithm. 
  \end{itemize}
\item For noise level $\epsilon=2$, we observe in Figure~\ref{fig:partial_ordering_offline_2}: 
\begin{itemize}
    \item For j10 and j20, the partial ordering is almost the same as for the $\epsilon=1$ level, although the flipping already occurs at j20, where the $\stnu$ becomes faster than the $\saa$.
    \item Surprisingly, the pattern changes for the ubo50 and ubo100 instances. We observe that the best performance on time offline is obtained by the $\stnu$ according to the Wilcoxon. We found that this is mainly caused by the much higher feasibility ratio of the $\stnu$ method. 
    \item Furthermore, we find that the $\proactive_{0.9}$ becomes better than $\reactive$ because of the higher feasibility ratio again (as the two methods have the same offline procedure). 
    \item Again, for the larger instances sets (j20-ubo100), the $\saa$ has the worst relative computation time offline.
\end{itemize}
\end{itemize}

\begin{figure}[htb!]
    \centering
\begin{tikzpicture}
\begin{scope}[every node/.style={}]
    \node (stnu) at (5,0.6) {\stnu};
    \node (rea) at (0,0) {\reactive};
    \node (saa) at (3,0.5) {\saa};
    \node (pro) at (0,1) {\proheur};
    \node (instance) at (9,0) {j10};
\end{scope}

\begin{scope}[>={Stealth[black]},
              every node/.style={},
              every edge/.style={draw=black}]
    \path [->] (rea) edge node {} (saa);
    \path [->] (pro) edge node {} (saa);
    \path [->] (saa) edge node[midway, above] {} (stnu);
    \path [->] (rea) edge[bend right=10] node {} (stnu);
    \path [->] (pro) edge[bend left=10] node {} (stnu); 
\end{scope}
\end{tikzpicture}

\begin{tikzpicture}
\begin{scope}[every node/.style={}]
    \node (stnu) at (5,0.6) {\stnu};
    \node (rea) at (0,0) {\reactive};
    \node (saa) at (3,0.5) {\saa};
    \node (pro) at (0,1) {\proheur};
    \node (instance) at (9,0) {j20};
\end{scope}

\begin{scope}[>={Stealth[black]},
              every node/.style={},
              every edge/.style={draw=black}]
    \path [->] (rea) edge node {} (saa);
    \path [->] (pro) edge node {} (saa);
    \path [->] (saa) edge node[midway, above] {} (stnu);
    \path [->] (rea) edge[bend right=10] node {} (stnu);
    \path [->] (pro) edge[bend left=10] node {} (stnu); 
\end{scope}
\end{tikzpicture}

\begin{tikzpicture}
\begin{scope}[every node/.style={}]
    \node (stnu) at (3,0.5) {\stnu};
    \node (rea) at (0,0) {\reactive};
    \node (saa) at (5,0.5) {\saa};
    \node (pro) at (0,1) {\proheur};
    \node (instance) at (9,0) {j30};
\end{scope}

\begin{scope}[>={Stealth[black]},
              every node/.style={},
              every edge/.style={draw=black}]
    \path [->] (rea) edge[bend right=10] node {} (saa);
    \path [->] (pro) edge[bend left=10]  node {} (saa);
    \path [->] (stnu) edge node {} (saa);
    \path [->] (rea) edge node {} (stnu);
    \path [->] (pro) edge node {} (stnu); 
\end{scope}
\end{tikzpicture}

\begin{tikzpicture}
\begin{scope}[every node/.style={}]
    \node (stnu) at (3,0.5) {\stnu};
    \node (rea) at (0,0) {\reactive};
    \node (saa) at (5,0.5) {\saa};
    \node (pro) at (0,1) {\proheur};
    \node (instance) at (9,0) {ubo50};
\end{scope}

\begin{scope}[>={Stealth[black]},
              every node/.style={},
              every edge/.style={draw=black}]
    \path [->] (rea) edge[bend right=10] node {} (saa);
    \path [->] (pro) edge[bend left=10]  node {} (saa);
    \path [->] (stnu) edge node {} (saa);
    \path [->] (rea) edge node {} (stnu);
    \path [->] (pro) edge node {} (stnu); 
\end{scope}
\end{tikzpicture}

\begin{tikzpicture}
\begin{scope}[every node/.style={}]
    \node (stnu) at (3,0.5) {\stnu};
    \node (rea) at (0,0) {\reactive};
    \node (saa) at (5,0.5) {\saa};
    \node (pro) at (0,1) {\proheur};
    \node (instance) at (9,0) {ubo100};
\end{scope}

\begin{scope}[>={Stealth[black]},
              every node/.style={},
              every edge/.style={draw=black}]
    \path [->] (rea) edge[bend right=10] node {} (saa);
    \path [->] (pro) edge[bend left=10]  node {} (saa);
    \path [->] (stnu) edge node {} (saa);
    \path [->] (rea) edge node {} (stnu);
    \path [->] (pro) edge node {} (stnu); 
     \path [->] (rea) edge node {} (pro); 
\end{scope}
\end{tikzpicture}
    \caption{Partial ordering time offline $\epsilon=1$. Solid arrow indicating a significant difference obtained by the Wilcoxon test, and dashed arrow only a proportional difference.}
    \label{fig:partial_ordering_offline_1}
\end{figure}

\begin{figure}[htb!]
    \centering
\begin{tikzpicture}
\begin{scope}[every node/.style={}]
    \node (stnu) at (5,0.6) {\stnu};
    \node (rea) at (0,0) {\reactive};
    \node (saa) at (3,0.5) {\saa};
    \node (pro) at (0,1) {\proheur};
    \node (instance) at (9,0) {j10};
\end{scope}

\begin{scope}[>={Stealth[black]},
              every node/.style={},
              every edge/.style={draw=black}]
    \path [->] (rea) edge node {} (saa);
    \path [->] (pro) edge node {} (saa);
    \path [->] (saa) edge node[midway, above] {} (stnu);
    \path [->] (rea) edge[bend right=10] node {} (stnu);
    \path [->] (pro) edge[bend left=10] node {} (stnu); 
\end{scope}
\end{tikzpicture}

\begin{tikzpicture}
\begin{scope}[every node/.style={}]
    \node (stnu) at (5,0.6) {\stnu};
    \node (rea) at (0,0) {\reactive};
    \node (saa) at (3,0.5) {\saa};
    \node (pro) at (0,1) {\proheur};
    \node (instance) at (9,0) {j20};
\end{scope}

\begin{scope}[>={Stealth[black]},
              every node/.style={},
              every edge/.style={draw=black}]
    \path [->] (rea) edge node {} (saa);
    \path [->] (pro) edge node {} (saa);
    \path [->] (saa) edge node[midway, above] {} (stnu);
    \path [->] (rea) edge[bend right=10] node {} (stnu);
    \path [->] (pro) edge[bend left=10] node {} (stnu); 
\end{scope}
\end{tikzpicture}

\begin{tikzpicture}
\begin{scope}[every node/.style={}]
    \node (stnu) at (3,0.5) {\stnu};
    \node (rea) at (0,0) {\reactive};
    \node (saa) at (5,0.5) {\saa};
    \node (pro) at (0,1) {\proheur};
    \node (instance) at (9,0) {j30};
\end{scope}

\begin{scope}[>={Stealth[black]},
              every node/.style={},
              every edge/.style={draw=black}]
    \path [->] (rea) edge[bend right=10] node {} (saa);
    \path [->] (pro) edge[bend left=10]  node {} (saa);
    \path [->] (stnu) edge node {} (saa);
    \path [->, dashed] (rea) edge node {} (stnu);
    \path [->] (pro) edge node {} (stnu); 
\end{scope}
\end{tikzpicture}

\begin{tikzpicture}
    
\begin{scope}[every node/.style={}]
    \node (A) at (0,0) {\stnu};
    \node (B) at (2,0) {\proheur};
    \node (C) at (4.5,0) {\reactive};
    \node (D) at (7,0) {\saa};
   \node (instance) at (9,0) {ubo50};
\end{scope}

\begin{scope}[>={Stealth[black]},
              every node/.style={},
              every edge/.style={draw=black}]
    \path [->] (A) edge node {} (B);
    \path [->] (A) edge[bend right=10] node {} (C);
    \path [->] (B) edge node {} (C);
    \path [->] (B) edge[bend left=10] node {} (D);
    \path [->] (C) edge node {} (D);
        \path [->] (A) edge[bend right=10] node {} (D);
\end{scope}
\end{tikzpicture}

\begin{tikzpicture}
    
\begin{scope}[every node/.style={}]
    \node (A) at (0,0) {\stnu};
    \node (B) at (2,0) {\proheur};
    \node (C) at (4.5,0) {\reactive};
    \node (D) at (7,0) {\saa};
   \node (instance) at (9,0) {ubo100};
\end{scope}

\begin{scope}[>={Stealth[black]},
              every node/.style={},
              every edge/.style={draw=black}]
    \path [->] (A) edge node {} (B);
    \path [->] (A) edge[bend right=10] node {} (C);
    \path [->] (B) edge node {} (C);
    \path [->] (B) edge[bend left=10] node {} (D);
    \path [->,dashed] (C) edge node {} (D);
        \path [->] (A) edge[bend right=10] node {} (D);
\end{scope}
\end{tikzpicture}
    \caption{Partial ordering time offline $\epsilon=2$. The solid arrow indicates a significant difference obtained by the Wilcoxon test, and the dashed arrow shows only a proportional difference.}
    \label{fig:partial_ordering_offline_2}
\end{figure}

\subsubsection{Partial Ordering Time Online}
We observe the results from the tests on online computation time:
\begin{itemize}
    \item For $\epsilon=1$ (Figure~\ref{fig:part_ord_online_2}), we find the same partial ordering for each instance set, which is also in line with our expectations: 
    \begin{itemize}
        \item There is no difference between $\saa$ and $\proactive_{0.9}$ as both only require a feasibility check only. 
        \item The $\stnu$ real-time execution algorithm turns out to be more efficient online than the reactive method, which comprises multiple solver calls, this is also expected.
    \end{itemize}
    \item For $\epsilon=2$ (Figure~\ref{fig:part_ord_online_2}), we find the same pattern as for $\epsilon=1$; for problem sets j10-j30. However, we find that for ubo50 and ubo100, the $\stnu$  starts to outperform other methods, explained by the higher feasibility ratio of the \stnu. 
\end{itemize}

\begin{figure}[htb!]
    \centering
   \begin{tikzpicture}
\begin{scope}[every node/.style={}]
    \node (rea) at (6,0.5) {\reactive};
    \node (saa) at (0,0) {\saa};
    \node (stnu) at (3,0.5) {\stnu};
    \node (pro) at (0,1) {\proheur};
     \node (instances) at (9,0.5) {j10-ubo100};
\end{scope}

\begin{scope}[>={Stealth[black]},
              every node/.style={},
              every edge/.style={draw=black}]
    \path [->] (saa) edge node {} (stnu);
    \path [->] (pro) edge node {} (stnu);
   \path [->] (stnu) edge node {} (rea);
   \path [->] (pro) edge node {} (rea);
   \path [->] (saa) edge node {} (rea);
\end{scope}
\end{tikzpicture}
    \caption{Partial ordering time online $\epsilon=1$, j10, j20, j30, ubo50, and ubo100. Solid arrow indicating a significant difference obtained by the Wilcoxon test, and dashed arrow only a proportional difference.}
    \label{fig:part_ord_online_1}
\end{figure}

\begin{figure}[htb!]
    \centering
   \begin{tikzpicture}
\begin{scope}[every node/.style={}]
    \node (rea) at (6,0.5) {\reactive};
    \node (saa) at (0,0) {\saa};
    \node (stnu) at (3,0.5) {\stnu};
    \node (pro) at (0,1) {\proheur};
     \node (instances) at (9,0.5) {j10-j30};
\end{scope}

\begin{scope}[>={Stealth[black]},
              every node/.style={},
              every edge/.style={draw=black}]
    \path [->] (saa) edge node {} (stnu);
    \path [->] (pro) edge node {} (stnu);
   \path [->] (stnu) edge node {} (rea);
   \path [->] (pro) edge node {} (rea);
   \path [->] (saa) edge node {} (rea);
\end{scope}
\end{tikzpicture}

   \begin{tikzpicture}
\begin{scope}[every node/.style={}]
    \node (rea) at (6,0.5) {\reactive};
    \node (saa) at (3,0) {\saa};
    \node (stnu) at (0,0.5) {\stnu};
    \node (pro) at (3,1) {\proheur};
     \node (instances) at (9,0.5) {ubo50-ubo100};
\end{scope}

\begin{scope}[>={Stealth[black]},
              every node/.style={},
              every edge/.style={draw=black}]
    \path [->] (stnu) edge node {} (saa);
    \path [->] (stnu) edge node {} (pro);
   \path [->] (stnu) edge node {} (rea);
   \path [->] (pro) edge node {} (rea);
   \path [->] (saa) edge node {} (rea);
\end{scope}
\end{tikzpicture}
  \caption{Partial ordering time online $\epsilon=2$, ubo50, and ubo100. Solid arrow indicating a significant difference obtained by the Wilcoxon test, and dashed arrow only a proportional difference.}
    \label{fig:part_ord_online_2}
\end{figure}

\subsection{Magnitude Tests}\label{app:magnitude}
In general, we observed that for the pairs of methods for which we found a significant partial ordering, we also found a significant difference in the magnitude of the different metrics. The test results for solution quality are presented in Table \ref{tab:obj_magnitude_1} ($\epsilon=1$) and \ref{tab:obj_magnitude_2} ($\epsilon=2$), for time offline in Table \ref{tab:offline_magnitude_1} ($\epsilon=1$) and \ref{tab:offline_magnitude_2} ($\epsilon=2$). 

The remarkable things that we observed in the results of the magnitude tests are highlighted per performance metric in the following subsections.

\subsubsection{Magnitude Tests Solution Quality}
In Table \ref{tab:obj_magnitude_1} and Table \ref{tab:obj_magnitude_2}, a comparison of the magnitude differences in solution quality (makespan) on the double hits is provided. Here, the normalized averages indicate how much this difference is. We see that the \stnu \ method for all pair-wise comparisons on the double hits results in the lowest normalized average of the makespan. In the comparison of $\reactive$ and $\proactive$ we see lower normalized averages for the $\reactive$ approach. Furthermore, we observe that for $\epsilon=2$ the magnitude differences become larger for the different pairs.

\begin{table*}[htb!]
    \centering
     \scalebox{0.75}{
    \begin{tabular}{l|llllll} \hline
    j10 & \stnu-$\react$ & \stnu-$\saashort$ & \stnu-$\pro_{0.9}$ & $\react$-$\saashort$ & $\react$-$\pro_{0.9}$ & $\saashort$-$\pro_{0.9}$ \\
    $[n] \ t \ (p)$  & [260] -9.367 (*)  & [260] -13.762 (*)  & [260] -13.703 (*)  & [340] -10.114 (*)  & [340] -10.542 (*)  & [340] -5.637 (*)  \\
     norm. avg. & $\stnu$: 0.974 & $\stnu$: 0.962 & $\stnu$: 0.962 & $\react$: 0.986 & $\react$: 0.985 & $\saashort$: 0.999 \\
   norm. avg.  & $\react$: 1.026 & $\saashort$: 1.038 & $\pro_{0.9}$: 1.038 & $\saashort$: 1.014 & $\pro_{0.9}$: 1.015 & $\pro_{0.9}$: 1.001 \\ \hline
    j20 & \stnu-$\react$ & \stnu-$\saashort$ & \stnu-$\pro_{0.9}$ & $\react$-$\saashort$ & $\react$-$\pro_{0.9}$ & $\saashort$-$\pro_{0.9}$ \\
    $[n] \ t \ (p)$ & [230] -7.62 (*)  & [230] -8.608 (*)  & [230] -8.992 (*)  & [270] -6.868 (*)  & [270] -7.713 (*)  & [270] -5.82 (*)  \\
    norm. avg. & $\stnu$: 0.982 & $\stnu$: 0.978 & $\stnu$: 0.977 & $\react$: 0.995 & $\react$: 0.993 & $\saashort$: 0.998 \\
    norm. avg. & $\react$: 1.018 & $\saashort$: 1.022 & $\pro_{0.9}$: 1.023 & $\saashort$: 1.005 & $\pro_{0.9}$: 1.007 & $\pro_{0.9}$: 1.002 \\ \hline
    j30 & \stnu-$\react$ & \stnu-$\saashort$ & \stnu-$\pro_{0.9}$ & $\react$-$\saashort$ & $\react$-$\pro_{0.9}$ & $\saashort$-$\pro_{0.9}$ \\
   $[n] \ t \ (p)$  & [280] -4.072 (*)  & [280] -5.946 (*)  & [280] -7.06 (*)  & [320] -5.698 (*)  & [320] -9.422 (*)  & [320] -6.226 (*)  \\
    norm. avg. & $\stnu$: 0.993 & $\stnu$: 0.99 & $\stnu$: 0.988 & $\react$: 0.997 & $\react$: 0.995 & $\saashort$: 0.998 \\
    norm. avg. & $\react$: 1.007 & $\saashort$: 1.01 & $\pro_{0.9}$: 1.012 & $\saashort$: 1.003 & $\pro_{0.9}$: 1.005 & $\pro_{0.9}$: 1.002 \\ \hline
    ubo50 & \stnu-$\react$ & \stnu-$\saashort$ & \stnu-$\pro_{0.9}$ & $\react$-$\saashort$ & $\react$-$\pro_{0.9}$ & $\saashort$-$\pro_{0.9}$ \\
    $[n] \ t \ (p)$  & [330] -11.359 (*)  & [330] -11.391 (*)  & [330] -11.346 (*)  & [370] -2.858 (*)  & [370] -6.734 (*)  & [370] -6.014 (*)  \\
   norm. avg.  & $\stnu$: 0.985 & $\stnu$: 0.985 & $\stnu$: 0.984 & $\react$: 1.0 & $\react$: 0.999 & $\saashort$: 0.999 \\
   norm. avg.  & $\react$: 1.015 & $\saashort$: 1.015 & $\pro_{0.9}$: 1.016 & $\saashort$: 1.0 & $\pro_{0.9}$: 1.001 & $\pro_{0.9}$: 1.001 \\ \hline
    ubo100 & \stnu-$\react$ & \stnu-$\saashort$ & \stnu-$\pro_{0.9}$ & $\react$-$\saashort$ & $\react$-$\pro_{0.9}$ & $\saashort$-$\pro_{0.9}$ \\
    $[n] \ t \ (p)$ & [358] -6.316 (*)  & [370] -6.813 (*)  & [370] -6.817 (*)  & [388] -2.833 (*)  & [388] -7.738 (*)  & [400] -0.514 (0.608) \\
   norm. avg.  & $\stnu$: 0.991 & $\stnu$: 0.99 & $\stnu$: 0.99 & $\react$: 0.999 & $\react$: 0.999 & $\saashort$: 1.0 \\
    norm. avg. & $\react$: 1.009 & $\saashort$: 1.01 & $\pro_{0.9}$: 1.01 & $\saashort$: 1.001 & $\pro_{0.9}$: 1.001 & $\pro_{0.9}$: 1.0 \\
    \hline
    \end{tabular}}
    \caption{Magnitude test on solution quality for noise factor $\epsilon=1$. Using a pairwise t-test, including all instances for which both methods found a feasible solution. Each cell  shows on the first row [nr pairs] t-stat (p-value) with (*) for p $<$ 0.05 and on the second row the normalized average of method 1 and on the third row the normalized average of method 2.}
    \label{tab:obj_magnitude_1}
\end{table*}

\begin{table*}[htb!]
    \centering
     \scalebox{0.75}{
    \begin{tabular}{l|llllll} \hline
    j10 & $\stnu$-$\react$ & $\stnu$-$\saashort$ & $\stnu$-$\pro_{0.9}$ & $\react$-$\saashort$ & $\react$-$\pro_{0.9}$ & $\saashort$-$\pro_{0.9}$ \\
    $[n] \ t \ (p)$ & [205] -9.128 (*)  & [204] -16.872 (*)  & [206] -17.395 (*)  & [230] -10.858 (*)  & [233] -11.811 (*)  & [231] -4.856 (*)  \\
    norm. avg. & $\stnu$: 0.963 & $\stnu$: 0.929 & $\stnu$: 0.925 & $\react$: 0.966 & $\react$: 0.964 & $\saashort$: 0.997 \\
    norm. avg. & $\react$: 1.037 & $\saashort$: 1.071 & $\pro_{0.9}$: 1.075 & $\saashort$: 1.034 & $\pro_{0.9}$: 1.036 & $\pro_{0.9}$: 1.003 \\ \hline 
    j20 & $\stnu$-$\react$ & $\stnu$-$\saashort$ & $\stnu$-$\pro_{0.9}$ & $\react$-$\saashort$ & $\react$-$\pro_{0.9}$ & $\saashort$-$\pro_{0.9}$ \\
    $[n] \ t \ (p)$ & [170] -10.855 (*)  & [172] -11.237 (*)  & [170] -11.89 (*)  & [174] -4.95 (*)  & [176] -6.22 (*)  & [173] -3.656 (*)  \\
   norm. avg.  & $\stnu$: 0.972 & $\stnu$: 0.968 & $\stnu$: 0.964 & $\react$: 0.991 & $\react$: 0.989 & $\saashort$: 0.997 \\
     norm. avg. & $\react$: 1.028 & $\saashort$: 1.032 & $\pro_{0.9}$: 1.036 & $\saashort$: 1.009 & $\pro_{0.9}$: 1.011 & $\pro_{0.9}$: 1.003 \\ \hline
    j30 & $\stnu$-$\react$ & $\stnu$-$\saashort$ & $\stnu$-$\pro_{0.9}$ & $\react$-$\saashort$ & $\react$-$\pro_{0.9}$ & $\saashort$-$\pro_{0.9}$ \\
    $[n] \ t \ (p)$ & [140] -4.151 (*)  & [148] -6.017 (*)  & [138] -6.149 (*)  & [152] -4.199 (*)  & [160] -6.746 (*)  & [150] -5.404 (*)  \\
    norm. avg. & $\stnu$: 0.99 & $\stnu$: 0.986 & $\stnu$: 0.985 & $\react$: 0.994 & $\react$: 0.991 & $\saashort$: 0.996 \\
   norm. avg.  & $\react$: 1.01 & $\saashort$: 1.014 & $\pro_{0.9}$: 1.015 & $\saashort$: 1.006 & $\pro_{0.9}$: 1.009 & $\pro_{0.9}$: 1.004 \\ \hline
    ubo50 & $\stnu$-$\react$ & $\stnu$-$\saashort$ & $\stnu$-$\pro_{0.9}$ & $\react$-$\saashort$ & $\react$-$\pro_{0.9}$ & $\saashort$-$\pro_{0.9}$ \\
    $[n] \ t \ (p)$ & [141] -7.974 (*)  & [146] -8.139 (*)  & [148] -8.125 (*)  & [145] -3.581 (*)  & [155] -4.646 (*)  & [150] -2.823 (*)  \\
    norm. avg. & $\stnu$: 0.984 & $\stnu$: 0.981 & $\stnu$: 0.979 & $\react$: 0.998 & $\react$: 0.997 & $\saashort$: 0.999 \\
    norm. avg. & $\react$: 1.016 & $\saashort$: 1.019 & $\pro_{0.9}$: 1.021 & $\saashort$: 1.002 & $\pro_{0.9}$: 1.003 & $\pro_{0.9}$: 1.001 \\ \hline
    ubo100 & $\stnu$-$\react$ & $\stnu$-$\saashort$ & $\stnu$-$\pro_{0.9}$ & $\react$-$\saashort$ & $\react$-$\pro_{0.9}$ & $\saashort$-$\pro_{0.9}$ \\
    $[n] \ t \ (p)$ & [94] -0.305 (0.761) & [114] -0.296 (0.767) & [114] -0.93 (0.355) & [76] -1.573 (0.12) & [99] -4.187 (*)  & [94] 0.59 (0.556) \\
    norm. avg. & $\stnu$: 0.999 & $\stnu$: 0.999 & $\stnu$: 0.998 & $\react$: 0.999 & $\react$: 0.999 & $\saashort$: 1.0 \\
    norm. avg. & $\react$: 1.001 & $\saashort$: 1.001 & $\pro_{0.9}$: 1.002 & $\saashort$: 1.001 & $\pro_{0.9}$: 1.001 & $\pro_{0.9}$: 1.0 \\
    \hline
    \end{tabular}}
    \caption{Magnitude test on solution quality for noise factor $\epsilon=2$. Using a pairwise t-test, including all instances for which both methods found a feasible solution. Each cell shows on the first row [nr pairs] t-stat (p-value) with (*) for p $<$ 0.05 and on the second row the normalized average of method 1 and on the third row the normalized average of method 2.}
    \label{tab:obj_magnitude_2}
\end{table*}

\newpage
\subsubsection{Magnitude Tests Time Offline}
 For time offline, $\epsilon=2$, ubo50 and ubo100, the Wilcoxon test found that the $\stnu$ is better than e.g. $\proactive_{0.9}$ and $\reactive$ while looking at the magnitude results on double hits (this is an important difference in test procedure!), we find that the $\proactive_{0.9}$ and $\reactive$ are much faster offline than $\stnu$, which is also expected. Then, we observe no difference between $\proactive_{0.9}$ and $\reactive$, which can be explained by the fact that both have the same offline procedure, and we select only double hits for the test.
\begin{table*}[htb!]
    \centering
     \scalebox{0.75}{
    \begin{tabular}{l|llllll} \hline
    j10 & $\pro_{0.9}$-$\react$ & $\react$-$\stnu$ & $\react$-$\saashort$ & $\pro_{0.9}$-$\stnu$ & $\pro_{0.9}$-$\saashort$ & $\saashort$-$\stnu$ \\
    $[n] \ t \ (p)$ & [340] nan (nan) & [260] -365.854 (*)  & [340] -77.22 (*)  & [260] -365.854 (*)  & [340] -77.22 (*)  & [260] -19.399 (*)  \\
    norm. avg. & $\pro_{0.9}$: 1.0 & $\react$: 0.14 & $\react$: 0.32 & $\pro_{0.9}$: 0.14 & $\pro_{0.9}$: 0.32 & $\saashort$: 0.65 \\
    norm. avg. & $\react$: 1.0 & $\stnu$: 1.86 & $\saashort$: 1.68 & $\stnu$: 1.86 & $\saashort$: 1.68 & $\stnu$: 1.35 \\ \hline
    j20 & $\pro_{0.9}$-$\react$ & $\react$-$\stnu$ & $\react$-$\saashort$ & $\pro_{0.9}$-$\stnu$ & $\pro_{0.9}$-$\saashort$ & $\saashort$-$\stnu$ \\
   $[n] \ t \ (p)$  & [270] nan (nan) & [230] -45.769 (*)  & [270] -52.707 (*)  & [230] -45.769 (*)  & [270] -52.707 (*)  & [230] -3.062 (*)  \\
    norm. avg. & $\pro_{0.9}$: 1.0 & $\react$: 0.22 & $\react$: 0.28 & $\pro_{0.9}$: 0.22 & $\pro_{0.9}$: 0.28 & $\saashort$: 0.91 \\
    norm. avg. & $\react$: 1.0 & $\stnu$: 1.78 & $\saashort$: 1.72 & $\stnu$: 1.78 & $\saashort$: 1.72 & $\stnu$: 1.09 \\ \hline
    j30 & $\pro_{0.9}$-$\react$ & $\react$-$\stnu$ & $\react$-$\saashort$ & $\pro_{0.9}$-$\stnu$ & $\pro_{0.9}$-$\saashort$ & \stnu-$\saashort$ \\
   $[n] \ t \ (p)$  & [320] nan (nan) & [280] -38.703 (*)  & [320] -53.836 (*)  & [280] -38.703 (*)  & [320] -53.836 (*)  & [280] 0.5 (0.618) \\
    norm. avg. & $\pro_{0.9}$: 1.0 & $\react$: 0.3 & $\react$: 0.32 & $\pro_{0.9}$: 0.3 & $\pro_{0.9}$: 0.32 & $\stnu$: 1.02 \\
   norm. avg.  & $\react$: 1.0 & $\stnu$: 1.7 & $\saashort$: 1.68 & $\stnu$: 1.7 & $\saashort$: 1.68 & $\saashort$: 0.98 \\ \hline
    ubo50 & $\pro_{0.9}$-$\react$ & $\react$-$\stnu$ & $\react$-$\saashort$ & $\pro_{0.9}$-$\stnu$ & $\pro_{0.9}$-$\saashort$ & \stnu-$\saashort$ \\
    $[n] \ t \ (p)$ & [370] nan (nan) & [330] -36.15 (*)  & [370] -72.563 (*)  & [330] -36.15 (*)  & [370] -72.563 (*)  & [330] -5.636 (*)  \\
   norm. avg.  & $\pro_{0.9}$: 1.0 & $\react$: 0.36 & $\react$: 0.24 & $\pro_{0.9}$: 0.36 & $\pro_{0.9}$: 0.24 & $\stnu$: 0.82 \\
   norm. avg.  & $\react$: 1.0 & $\stnu$: 1.64 & $\saashort$: 1.76 & $\stnu$: 1.64 & $\saashort$: 1.76 & $\saashort$: 1.18 \\ \hline
    ubo100 & $\pro_{0.9}$-$\react$ & $\react$-$\stnu$ & $\react$-$\saashort$ & $\pro_{0.9}$-$\stnu$ & $\pro_{0.9}$-$\saashort$ & \stnu-$\saashort$ \\
   $[n] \ t \ (p)$  & [388] nan (nan) & [358] -27.209 (*)  & [388] -70.349 (*)  & [370] -25.293 (*)  & [400] -70.926 (*)  & [370] -6.391 (*)  \\
   norm. avg.  & $\pro_{0.9}$: 1.0 & $\react$: 0.43 & $\react$: 0.27 & $\pro_{0.9}$: 0.45 & $\pro_{0.9}$: 0.27 & $\stnu$: 0.79 \\
   norm. avg.  & $\react$: 1.0 & $\stnu$: 1.57 & $\saashort$: 1.73 & $\stnu$: 1.55 & $\saashort$: 1.73 & $\saashort$: 1.21 \\
    \hline
    \end{tabular}}
    \caption{Magnitude test on time offline for noise factor $\epsilon=1$. Using a pairwise t-test, including all instances for which both methods found a feasible solution. Each cell  shows on the first row [nr pairs] t-stat (p-value) with (*) for p $<$ 0.05 and on the second row the normalized average of method 1 and on the third row the normalized average of method 2.}
    \label{tab:offline_magnitude_1}
\end{table*}

\begin{table*}[htb!]
    \centering
     \scalebox{0.75}{
    \begin{tabular}{l|llllll} \hline
    j10 & $\pro_{0.9}$-$\react$ & $\react$-$\stnu$ & $\react$-$\saashort$ & $\pro_{0.9}$-$\stnu$ & $\pro_{0.9}$-$\saashort$ & $\saashort$-$\stnu$ \\
     $[n] \ t \ (p)$ & [233] nan (nan) & [205] -229.689 (*)  & [230] -59.39 (*)  & [206] -230.812 (*)  & [231] -59.404 (*)  & [204] -17.066 (*)  \\
    norm. avg. & $\pro_{0.9}$: 1.0 & $\react$: 0.15 & $\react$: 0.35 & $\pro_{0.9}$: 0.15 & $\pro_{0.9}$: 0.35 & $\saashort$: 0.61 \\
    norm. avg. & $\react$: 1.0 & $\stnu$: 1.85 & $\saashort$: 1.65 & $\stnu$: 1.85 & $\saashort$: 1.65 & $\stnu$: 1.39 \\ \hline
    j20 & $\pro_{0.9}$-$\react$ & $\react$-$\stnu$ & $\react$-$\saashort$ & $\pro_{0.9}$-$\stnu$ & $\pro_{0.9}$-$\saashort$ & \stnu-$\saashort$ \\
    $[n] \ t \ (p)$ & [176] nan (nan) & [170] -82.607 (*)  & [174] -63.732 (*)  & [170] -82.607 (*)  & [173] -63.382 (*)  & [172] 0.379 (0.705) \\
   norm. avg.  & $\pro_{0.9}$: 1.0 & $\react$: 0.21 & $\react$: 0.23 & $\pro_{0.9}$: 0.21 & $\pro_{0.9}$: 0.23 & $\stnu$: 1.01 \\
    norm. avg. & $\react$: 1.0 & $\stnu$: 1.79 & $\saashort$: 1.77 & $\stnu$: 1.79 & $\saashort$: 1.77 & $\saashort$: 0.99 \\ \hline
    j30 & $\pro_{0.9}$-$\react$ & $\react$-$\stnu$ & $\react$-$\saashort$ & $\pro_{0.9}$-$\stnu$ & $\pro_{0.9}$-$\saashort$ & \stnu-$\saashort$ \\
    $[n] \ t \ (p)$ & [160] nan (nan) & [140] -33.765 (*)  & [152] -40.336 (*)  & [138] -33.253 (*)  & [150] -39.904 (*)  & [148] -2.168 (*)  \\
   norm. avg.  & $\pro_{0.9}$: 1.0 & $\react$: 0.34 & $\react$: 0.33 & $\pro_{0.9}$: 0.34 & $\pro_{0.9}$: 0.33 & $\stnu$: 0.9 \\
   norm. avg.  & $\react$: 1.0 & $\stnu$: 1.66 & $\saashort$: 1.67 & $\stnu$: 1.66 & $\saashort$: 1.67 & $\saashort$: 1.1 \\ \hline
    ubo50 & \stnu-$\pro_{0.9}$ & \stnu-$\react$ & \stnu-$\saashort$ & $\pro_{0.9}$-$\react$ & $\pro_{0.9}$-$\saashort$ & $\react$-$\saashort$ \\
    $[n] \ t \ (p)$ & [148] 26.653 (*)  & [141] 27.168 (*)  & [146] -0.952 (0.343) & [155] nan (nan) & [150] -45.457 (*)  & [145] -44.29 (*)  \\
   norm. avg.  & $\stnu$: 1.7 & $\stnu$: 1.72 & $\stnu$: 0.95 & $\pro_{0.9}$: 1.0 & $\pro_{0.9}$: 0.24 & $\react$: 0.24 \\
    norm. avg. & $\pro_{0.9}$: 0.3 & $\react$: 0.28 & $\saashort$: 1.05 & $\react$: 1.0 & $\saashort$: 1.76 & $\saashort$: 1.76 \\ \hline
    ubo100 & \stnu-$\pro_{0.9}$ & \stnu-$\react$ & \stnu-$\saashort$ & $\pro_{0.9}$-$\react$ & $\pro_{0.9}$-$\saashort$ & $\react$-$\saashort$ \\
    $[n] \ t \ (p)$ & [114] 14.765 (*)  & [94] 15.183 (*)  & [114] -3.125 (*)  & [99] nan (nan) & [94] -42.151 (*)  & [76] -42.944 (*)  \\
    norm. avg. & $\stnu$: 1.55 & $\stnu$: 1.61 & $\stnu$: 0.83 & $\pro_{0.9}$: 1.0 & $\pro_{0.9}$: 0.21 & $\react$: 0.18 \\
    norm. avg. & $\pro_{0.9}$: 0.45 & $\react$: 0.39 & $\saashort$: 1.17 & $\react$: 1.0 & $\saashort$: 1.79 & $\saashort$: 1.82 \\
    \hline
    \end{tabular}}
    \caption{Magnitude test on time offline for noise factor $\epsilon=2$. Using a pairwise t-test, including all instances for which both methods found a feasible solution. Each cell  shows on the first row [nr pairs] t-stat (p-value) with (*) for p $<$ 0.05 and on the second row the normalized average of method 1 and on the third row the normalized average of method 2.}
    \label{tab:offline_magnitude_2}
\end{table*}

\newpage
\subsubsection{Magnitude Tests Time Online}
Looking at the magnitude test of time online with $\epsilon=2$, according to Wilcoxon, $\stnu$ is outperforming $\saa$ and $\proactive_{0.9}$ for ubo50 and ubo100. However, this is again mainly caused by the higher ratio of feasible solutions, as when we analyze the magnitude difference the proactive methods are much faster online looking at the double hits (which is logical as this comprises only of feasibility checking). 
\begin{table*}[htb!]
    \centering
     \scalebox{0.75}{
    \begin{tabular}{l|llllll} \hline
    j10 & $\pro_{0.9}$-$\saashort$ & $\pro_{0.9}$-$\stnu$ & $\saashort$-$\stnu$ & $\pro_{0.9}$-$\react$ & $\saashort$-$\react$ & \stnu-$\react$ \\
    $[n] \ t \ (p)$  & [340] 0.679 (0.498) & [260] -109.014 (*)  & [260] -106.917 (*)  & [340] -4393.475 (*)  & [340] -4277.995 (*)  & [260] -1151.565 (*)  \\
   norm. avg.  & $\pro_{0.9}$: 1.02 & $\pro_{0.9}$: 0.05 & $\saashort$: 0.05 & $\pro_{0.9}$: 0.0 & $\saashort$: 0.0 & $\stnu$: 0.04 \\
    norm. avg. & $\saashort$: 0.98 & $\stnu$: 1.95 & $\stnu$: 1.95 & $\react$: 2.0 & $\react$: 2.0 & $\react$: 1.96 \\ \hline
    j20 & $\pro_{0.9}$-$\saashort$ & $\pro_{0.9}$-$\stnu$ & $\saashort$-$\stnu$ & $\pro_{0.9}$-$\react$ & $\saashort$-$\react$ & \stnu-$\react$ \\
    $[n] \ t \ (p)$ & [270] -0.663 (0.508) & [230] -285.563 (*)  & [230] -282.464 (*)  & [270] -7311.952 (*)  & [270] -7470.984 (*)  & [230] -605.38 (*)  \\
   norm. avg.  & $\pro_{0.9}$: 0.97 & $\pro_{0.9}$: 0.03 & $\saashort$: 0.03 & $\pro_{0.9}$: 0.0 & $\saashort$: 0.0 & $\stnu$: 0.07 \\
   norm. avg.  & $\saashort$: 1.03 & $\stnu$: 1.97 & $\stnu$: 1.97 & $\react$: 2.0 & $\react$: 2.0 & $\react$: 1.93 \\ \hline
    j30 & $\pro_{0.9}$-$\saashort$ & $\pro_{0.9}$-$\stnu$ & $\saashort$-$\stnu$ & $\pro_{0.9}$-$\react$ & $\saashort$-$\react$ & \stnu-$\react$ \\
    $[n] \ t \ (p)$ & [320] -0.156 (0.876) & [280] -789.246 (*)  & [280] -786.726 (*)  & [320] -14638.229 (*)  & [320] -14425.009 (*)  & [280] -338.664 (*)  \\
   norm. avg.  & $\pro_{0.9}$: 0.99 & $\pro_{0.9}$: 0.02 & $\saashort$: 0.02 & $\pro_{0.9}$: 0.0 & $\saashort$: 0.0 & $\stnu$: 0.09 \\
   norm. avg.  & $\saashort$: 1.01 & $\stnu$: 1.98 & $\stnu$: 1.98 & $\react$: 2.0 & $\react$: 2.0 & $\react$: 1.91 \\ \hline
    ubo50 & $\saashort$-$\pro_{0.9}$ & $\pro_{0.9}$-$\stnu$ & $\saashort$-$\stnu$ & $\pro_{0.9}$-$\react$ & $\saashort$-$\react$ & \stnu-$\react$ \\
   $[n] \ t \ (p)$  & [370] -0.795 (0.427) & [330] -3951.584 (*)  & [330] -3976.441 (*)  & [370] -30723.506 (*)  & [370] -31447.766 (*)  & [330] -142.758 (*)  \\
    norm. avg. & $\saashort$: 0.97 & $\pro_{0.9}$: 0.01 & $\saashort$: 0.01 & $\pro_{0.9}$: 0.0 & $\saashort$: 0.0 & $\stnu$: 0.15 \\
    norm. avg. & $\pro_{0.9}$: 1.03 & $\stnu$: 1.99 & $\stnu$: 1.99 & $\react$: 2.0 & $\react$: 2.0 & $\react$: 1.85 \\ \hline
    ubo100 & $\saashort$-$\pro_{0.9}$ & $\pro_{0.9}$-$\stnu$ & $\saashort$-$\stnu$ & $\pro_{0.9}$-$\react$ & $\saashort$-$\react$ & \stnu-$\react$ \\
    $[n] \ t \ (p)$ & [400] -0.065 (0.948) & [370] -56517.571 (*)  & [370] -48810.689 (*)  & [388] -83435.86 (*)  & [388] -78013.468 (*)  & [358] -60.202 (*)  \\
    norm. avg. & $\saashort$: 1.0 & $\pro_{0.9}$: 0.0 & $\saashort$: 0.0 & $\pro_{0.9}$: 0.0 & $\saashort$: 0.0 & $\stnu$: 0.23 \\
    norm. avg. & $\pro_{0.9}$: 1.0 & $\stnu$: 2.0 & $\stnu$: 2.0 & $\react$: 2.0 & $\react$: 2.0 & $\react$: 1.77 \\
    \hline
    \end{tabular}}
    \caption{Magnitude t-test on time online for noise factor $\epsilon=1$ (double hits). Each cell  shows on the first row [nr pairs] t-stat (p-value) with (*) for p $<$ 0.05 and on the second row the normalized average of method 1 and on the third row the normalized average of method 2.}
    \label{tab:online_magnitude_1}
\end{table*}

\begin{table*}[htb!]
    \centering
     \scalebox{0.75}{
    \begin{tabular}{l|llllll} \hline
    j10 & $\pro_{0.9}$-$\saashort$ & $\pro_{0.9}$-$\stnu$ & $\saashort$-$\stnu$ & $\pro_{0.9}$-$\react$ & $\saashort$-$\react$ & \stnu-$\react$ \\
    $[n] \ t \ (p)$ & [231] 0.391 (0.697) & [206] -92.546 (*)  & [204] -91.532 (*)  & [233] -4426.055 (*)  & [230] -4653.225 (*)  & [205] -1399.699 (*)  \\
    norm. avg. & $\pro_{0.9}$: 1.01 & $\pro_{0.9}$: 0.06 & $\saashort$: 0.06 & $\pro_{0.9}$: 0.0 & $\saashort$: 0.0 & $\stnu$: 0.03 \\
    norm. avg. & $\saashort$: 0.99 & $\stnu$: 1.94 & $\stnu$: 1.94 & $\react$: 2.0 & $\react$: 2.0 & $\react$: 1.97 \\ \hline
    j20 & $\pro_{0.9}$-$\saashort$ & $\pro_{0.9}$-$\stnu$ & $\saashort$-$\stnu$ & $\pro_{0.9}$-$\react$ & $\saashort$-$\react$ & \stnu-$\react$ \\
    $[n] \ t \ (p)$ & [173] 0.77 (0.443) & [170] -241.897 (*)  & [172] -262.612 (*)  & [176] -7787.819 (*)  & [174] -8817.279 (*)  & [170] -935.848 (*)  \\
    norm. avg. & $\pro_{0.9}$: 1.03 & $\pro_{0.9}$: 0.03 & $\saashort$: 0.03 & $\pro_{0.9}$: 0.0 & $\saashort$: 0.0 & $\stnu$: 0.06 \\
    norm. avg. & $\saashort$: 0.97 & $\stnu$: 1.97 & $\stnu$: 1.97 & $\react$: 2.0 & $\react$: 2.0 & $\react$: 1.94 \\ \hline
    j30 & $\saashort$-$\pro_{0.9}$ & $\pro_{0.9}$-$\stnu$ & $\saashort$-$\stnu$ & $\pro_{0.9}$-$\react$ & $\saashort$-$\react$ & \stnu-$\react$ \\
    $[n] \ t \ (p)$ & [150] -0.264 (0.792) & [138] -514.302 (*)  & [148] -547.413 (*)  & [160] -14437.976 (*)  & [152] -14871.238 (*)  & [140] -397.488 (*)  \\
    norm. avg. & $\saashort$: 0.99 & $\pro_{0.9}$: 0.02 & $\saashort$: 0.02 & $\pro_{0.9}$: 0.0 & $\saashort$: 0.0 & $\stnu$: 0.07 \\
   norm. avg.  & $\pro_{0.9}$: 1.01 & $\stnu$: 1.98 & $\stnu$: 1.98 & $\react$: 2.0 & $\react$: 2.0 & $\react$: 1.93 \\ \hline
    ubo50 & $\pro_{0.9}$-$\saashort$ & \stnu-$\pro_{0.9}$ & \stnu-$\saashort$ & $\pro_{0.9}$-$\react$ & $\saashort$-$\react$ & \stnu-$\react$ \\
    $[n] \ t \ (p)$ & [150] -0.707 (0.481) & [148] 2600.902 (*)  & [146] 2585.661 (*)  & [155] -25795.773 (*)  & [145] -23210.69 (*)  & [141] -121.392 (*)  \\
   norm. avg.  & $\pro_{0.9}$: 0.96 & $\stnu$: 1.99 & $\stnu$: 1.99 & $\pro_{0.9}$: 0.0 & $\saashort$: 0.0 & $\stnu$: 0.14 \\
   norm. avg.  & $\saashort$: 1.04 & $\pro_{0.9}$: 0.01 & $\saashort$: 0.01 & $\react$: 2.0 & $\react$: 2.0 & $\react$: 1.86 \\ \hline
    ubo100 & $\saashort$-$\pro_{0.9}$ & stnu-$\pro_{0.9}$ & \stnu-$\saashort$ & $\pro_{0.9}$-$\react$ & $\saashort$-$\react$ & \stnu-$\react$ \\
   $[n] \ t \ (p)$  & [94] 0.141 (0.888) & [114] 21743.706 (*)  & [114] 21911.639 (*)  & [99] -46429.42 (*)  & [76] -40238.643 (*)  & [94] -41.405 (*)  \\
   norm. avg.  & $\saashort$: 1.0 & $\stnu$: 2.0 & $\stnu$: 2.0 & $\pro_{0.9}$: 0.0 & $\saashort$: 0.0 & $\stnu$: 0.21 \\
   norm. avg.  & $\pro_{0.9}$: 1.0 & $\pro_{0.9}$: 0.0 & $\saashort$: 0.0 & $\react$: 2.0 & $\react$: 2.0 & $\react$: 1.79 \\
    \hline
    \end{tabular}}
    \caption{Magnitude t-test on time online for noise factor $\epsilon=2$ (double hits). Each cell  shows on the first row [nr pairs] t-stat (p-value) with (*) for p $<$ 0.05 and on the second row the normalized average of method 1 and on the third row the normalized average of method 2.}
    \label{tab:online_magnitude_2}
\end{table*}
\subsection{Summary}
In our main paper, we employed the statistical analysis from this document to provide a summarizing partial ordering per metric including a brief summary of the main findings. We decide to present the most occurring patterns and in the text, highlight any important exceptions.

\newpage
\section{Results Tables Partial Orderings}
\begin{table*}[htb!]
    \centering
     \scalebox{0.75}{
    \begin{tabular}{l|llllll} \hline
    j10 & $\stnu$-$\react$ & $\stnu$-$\saashort$ & $\stnu$-$\pro_{0.9}$ & $\react$-$\saashort$ & $\react$-$\pro_{0.9}$ & $\saashort$-$\pro_{0.9}$ \\
    Wilcoxon: [n] z (p) & [340] -1.659 (0.097) & [340] -0.067 (0.947) & [340] -0.082 (0.935) & [340] -10.571 (*) & [340] -10.918 (*) & [340] -5.756 (*) \\
proportion: [n] z (p)  & [308] 0.591 (*) & [310] 0.632 (*) & [310] 0.632 (*) & [117] 0.991 (*) & [121] 1.0 (*) & [37] 0.973 (*) \\ \hline
    j20 & $\stnu$-$\react$ & $\stnu$-$\saashort$ & $\stnu$-$\pro_{0.9}$ & $\react$-$\saashort$ & $\react$-$\pro_{0.9}$ & $\saashort$-$\pro_{0.9}$ \\
      Wilcoxon: [n] z (p) & [270] -1.523 (0.128) & [270] -2.401 (*) & [270] -2.679 (*) & [270] -5.936 (*) & [270] -7.852 (*) & [270] -6.144 (*) \\
    proportion: [n] z (p)  & [213] 0.596 (*) & [215] 0.633 (*) & [213] 0.643 (*) & [75] 0.827 (*) & [62] 1.0 (*) & [59] 0.898 (*) \\ \hline
    j30 & \stnu-$\react$ & \stnu-$\saashort$ & \stnu-$\pro_{0.9}$ & $\react$-$\saashort$ & $\react$-$\pro_{0.9}$ & $\saashort$-$\pro_{0.9}$ \\
     Wilcoxon: [n] z (p)  & [320] -0.892 (0.372) & [320] -1.218 (0.223) & [320] -2.157 (*) & [320] -5.54 (*) & [320] -9.395 (*) & [320] -5.299 (*) \\
    proportion: [n] z (p)  & [214] 0.579 (*) & [214] 0.575 (*) & [217] 0.618 (*) & [103] 0.757 (*) & [89] 1.0 (*) & [79] 0.797 (*) \\ \hline
    ubo50 & \stnu-$\react$ & \stnu-$\saashort$ & \stnu-$\pro_{0.9}$ & $\react$-$\saashort$ & $\react$-$\pro_{0.9}$ & $\saashort$-$\pro_{0.9}$ \\
      Wilcoxon: [n] z (p) & [370] -6.748 (*) & [370] -6.8 (*) & [370] -6.859 (*) & [370] -2.76 (*) & [370] -8.531 (*) & [370] -7.085 (*) \\
    proportion: [n] z (p)  & [277] 0.776 (*) & [276] 0.779 (*) & [278] 0.781 (*) & [61] 0.672 (*) & [73] 1.0 (*) & [65] 0.938 (*) \\ \hline
    ubo100 & \stnu-$\react$ & \stnu-$\saashort$ & \stnu-$\pro_{0.9}$ & $\react$-$\saashort$ & $\react$-$\pro_{0.9}$ & $\saashort$-$\pro_{0.9}$ \\
     Wilcoxon: [n] z (p)  & [400] -2.12 (*) & [400] -1.76 (0.078) & [400] -1.842 (0.065) & [400] -0.442 (0.659) & [400] -6.534 (*) & [400] -0.2 (0.842) \\
    proportion: [n] z (p)  & [312] 0.519 (0.533) & [314] 0.51 (0.778) & [311] 0.514 (0.65) & [147] 0.51 (0.869) & [88] 0.864 (*) & [151] 0.503 (1.0) \\
    \hline
    \end{tabular}}
    \caption{Pairwise comparison on schedule quality (makespan) for noise factor $\epsilon=1$. Using a Wilcoxon test and a proportion test. Including all instances for which at least one of the two methods found a feasible solution. Note that the ordering matters: the first method showed is the better of the two in each pair method 1 - method 2. Each cell shows on the first row [nr pairs] the z-value (p-value) of the Wilcoxon test with (*) for p  $<$ 0.05. Each cell shows on the second row [nr pairs] z-value (p-value) with (*) for p $<$ 0.05 for the proportion test.}
    \label{tab:obj_pairwise_1}
\end{table*}

\begin{table*}[htb!]
    \centering
     \scalebox{0.75}{
    \begin{tabular}{l|llllll} \hline
    j10 & $\stnu$-$\react$ & $\stnu$-$\saashort$ & $\stnu$-$\pro_{0.9}$ & $\react$-$\saashort$ & $\react$-$\pro_{0.9}$ & $\saashort$-$\pro_{0.9}$ \\
    Wilcoxon: [n] z (p) & [258] -5.194 (*) & [258] -7.825 (*) & [258] -7.822 (*) & [235] -10.082 (*) & [234] -11.046 (*) & [235] -4.058 (*) \\
   proportion: [n] z (p)  & [237] 0.696 (*) & [247] 0.802 (*) & [248] 0.802 (*) & [146] 0.911 (*) & [134] 0.993 (*) & [37] 0.838 (*) \\  \hline
    j20 & $\stnu$-$\react$ & $\stnu$-$\saashort$ & $\stnu$-$\pro_{0.9}$ & $\react$-$\saashort$ & $\react$-$\pro_{0.9}$ & $\saashort$-$\pro_{0.9}$ \\
     Wilcoxon: [n] z (p) & [217] -9.825 (*) & [217] -10.194 (*) & [216] -10.57 (*) & [182] -4.58 (*) & [177] -7.243 (*) & [182] -3.241 (*) \\
     proportion: [n] z (p) & [185] 0.854 (*) & [188] 0.872 (*) & [188] 0.888 (*) & [79] 0.759 (*) & [53] 1.0 (*) & [60] 0.7 (*) \\  \hline
    j30 & $\stnu$-$\react$ & $\stnu$-$\saashort$ & $\stnu$-$\pro_{0.9}$ & $\react$-$\saashort$ & $\react$-$\pro_{0.9}$ & $\saashort$-$\pro_{0.9}$ \\
     Wilcoxon: [n] z (p) & [233] -6.986 (*) & [234] -7.175 (*) & [232] -7.993 (*) & [183] -1.355 (0.175) & [163] -7.557 (*) & [182] -5.689 (*) \\
    proportion: [n] z (p) & [191] 0.749 (*) & [192] 0.776 (*) & [196] 0.796 (*) & [94] 0.574 (0.18) & [58] 1.0 (*) & [96] 0.802 (*) \\  \hline
    ubo50 & $\stnu$-$\react$ & $\stnu$-$\saashort$ & $\stnu$-$\pro_{0.9}$ & $\react$-$\saashort$ & $\react$-$\pro_{0.9}$ & $\saashort$-$\pro_{0.9}$ \\
   Wilcoxon: [n] z (p)   & [284] -11.76 (*) & [285] -11.73 (*) & [284] -11.799 (*) & [171] -1.973 (*) & [162] -3.74 (*) & [173] -0.938 (0.348) \\
    proportion: [n] z (p) & [236] 0.907 (*) & [238] 0.908 (*) & [236] 0.911 (*) & [50] 0.66 (*) & [39] 0.821 (*) & [36] 0.583 (0.405) \\  \hline
    ubo100 & $\stnu$-$\react$ & $\stnu$-$\saashort$ & $\stnu$-$\pro_{0.9}$ & $\react$-$\saashort$ & $\react$-$\pro_{0.9}$ & $\saashort$-$\pro_{0.9}$ \\
   Wilcoxon: [n] z (p)  & [285] -12.139 (*) & [284] -11.55 (*) & [288] -11.26 (*) & [141] -1.83 (0.067) & [122] -1.017 (0.309) & [146] -0.083 (0.934) \\
    proportion: [n] z (p) & [261] 0.851 (*) & [254] 0.799 (*) & [265] 0.796 (*) & [82] 0.415 (0.151) & [44] 0.477 (0.88) & [82] 0.5 (0.912) \\
    \hline
    \end{tabular}}
    \caption{Pairwise comparison on schedule quality (makespan) for noise factor $\epsilon=2$. Using a Wilcoxon test and a proportion test. Including all instances for which at least one of the two methods found a feasible solution. Note that the ordering matters: the first method showed is the better of the two in each pair method 1 - method 2. Each cell shows on the first row [nr pairs] the z-value (p-value) of the Wilcoxon test with (*) for p  $<$ 0.05. Each cell shows on the second row [nr pairs] z-value (p-value) with (*) for p $<$ 0.05 for the proportion test. }
    \label{tab:obj_pairwise_2}
\end{table*}

\begin{table*}[htb!]
    \centering
     \scalebox{0.75}{
    \begin{tabular}{l|llllll} \hline
    j10 & $\pro_{0.9}$-$\react$ & $\react$-$\stnu$ & $\react$-$\saashort$ & $\pro_{0.9}$-$\stnu$ & $\pro_{0.9}$-$\saashort$ & $\saashort$-$\stnu$ \\
   Wilcoxon: [n] z (p)  & [340] nan (nan) & [340] -15.982 (*) & [340] -15.982 (*) & [340] -15.982 (*) & [340] -15.982 (*) & [340] -13.986 (*) \\
    proportion: [n] z (p) & [nan] nan (nan) & [340] 1.0 (*) & [340] 1.0 (*) & [340] 1.0 (*) & [340] 1.0 (*) & [340] 0.882 (*) \\  \hline
    j20 & $\pro_{0.9}$-$\react$ & $\react$-$\stnu$ & $\react$-$\saashort$ & $\pro_{0.9}$-$\stnu$ & $\pro_{0.9}$-$\saashort$ & $\saashort$-$\stnu$ \\
      Wilcoxon: [n] z (p) & [270] nan (nan) & [270] -14.125 (*) & [270] -14.245 (*) & [270] -14.125 (*) & [270] -14.245 (*) & [270] -3.927 (*) \\
    proportion: [n] z (p) & [nan] nan (nan) & [270] 0.963 (*) & [270] 1.0 (*) & [270] 0.963 (*) & [270] 1.0 (*) & [270] 0.63 (*) \\  \hline
    j30 & $\pro_{0.9}$-$\react$ & $\react$-$\stnu$ & $\react$-$\saashort$ & $\pro_{0.9}$-$\stnu$ & $\pro_{0.9}$-$\saashort$ & \stnu-$\saashort$ \\
    Wilcoxon: [n] z (p) & [320] nan (nan) & [320] -13.963 (*) & [320] -15.506 (*) & [320] -13.963 (*) & [320] -15.506 (*) & [320] -1.63 (0.103) \\
     proportion: [n] z (p) & [nan] nan (nan) & [320] 0.969 (*) & [320] 1.0 (*) & [320] 0.969 (*) & [320] 1.0 (*) & [320] 0.5 (0.955) \\  \hline
    ubo50 & $\pro_{0.9}$-$\react$ & $\react$-$\stnu$ & $\react$-$\saashort$ & $\pro_{0.9}$-$\stnu$ & $\pro_{0.9}$-$\saashort$ & \stnu-$\saashort$ \\
    Wilcoxon: [n] z (p) & [370] nan (nan) & [370] -16.671 (*) & [370] -16.671 (*) & [370] -16.671 (*) & [370] -16.671 (*) & [370] -7.261 (*) \\
     proportion: [n] z (p)& [nan] nan (nan) & [370] 1.0 (*) & [370] 1.0 (*) & [370] 1.0 (*) & [370] 1.0 (*) & [370] 0.622 (*) \\  \hline
    ubo100 & $\pro_{0.9}$-$\react$ & $\react$-$\stnu$ & $\react$-$\saashort$ & $\pro_{0.9}$-$\stnu$ & $\pro_{0.9}$-$\saashort$ & \stnu-$\saashort$ \\
    Wilcoxon: [n] z (p) & [400] -3.463 (*) & [400] -14.988 (*) & [400] -15.286 (*) & [400] -15.969 (*) & [400] -17.332 (*) & [400] -6.406 (*) \\
    proportion: [n] z (p) & [12] 1.0 (*) & [400] 0.962 (*) & [400] 0.97 (*) & [400] 0.975 (*) & [400] 1.0 (*) & [400] 0.6 (*) \\
    \hline
    \end{tabular}}
    \caption{Pairwise comparison on time offline for noise factor $\epsilon=1$. Using a Wilcoxon test and a proportion test. Including all instances for which at least one of the two methods found a feasible solution. Note that the ordering matters: the first method showed is the better of the two in each pair method 1 - method 2 according to Wilcoxon. Each cell shows on the first row [nr pairs] the z-value (p-value) of the Wilcoxon test with (*) for p  $<$ 0.05. Each cell shows on the second row [nr pairs] the ratio of wins (p-value) with (*) for p $<$ 0.05.}
    \label{tab:offline_pairwise_1}
\end{table*}

\begin{table*}[htb!]
    \centering
     \scalebox{0.75}{
    \begin{tabular}{l|llllll} \hline
    j10 & $\pro_{0.9}$-$\react$ & $\react$-$\stnu$ & $\react$-$\saashort$ & $\pro_{0.9}$-$\stnu$ & $\pro_{0.9}$-$\saashort$ & $\saashort$-$\stnu$ \\
    Wilcoxon: [n] z (p)  & [234] -0.996 (0.319) & [258] -9.382 (*) & [235] -12.848 (*) & [258] -9.554 (*) & [235] -13.069 (*) & [258] -6.664 (*) \\
    proportion: [n] z (p) & [nan] nan (nan) & [258] 0.903 (*) & [235] 0.991 (*) & [258] 0.907 (*) & [235] 0.996 (*) & [258] 0.764 (*) \\  \hline
    j20 & $\pro_{0.9}$-$\react$ & $\react$-$\stnu$ & $\react$-$\saashort$ & $\pro_{0.9}$-$\stnu$ & $\pro_{0.9}$-$\saashort$ & \stnu-$\saashort$ \\
    Wilcoxon: [n] z (p) & [177] -0.994 (0.32) & [217] -4.543 (*) & [182] -10.458 (*) & [216] -4.456 (*) & [182] -10.213 (*) & [217] -5.281 (*) \\
     proportion: [n] z (p)& [nan] nan (nan) & [217] 0.816 (*) & [182] 0.973 (*) & [216] 0.815 (*) & [182] 0.967 (*) & [217] 0.604 (*) \\  \hline
    j30 & $\pro_{0.9}$-$\react$ & $\react$-$\stnu$ & $\react$-$\saashort$ & $\pro_{0.9}$-$\stnu$ & $\pro_{0.9}$-$\saashort$ & \stnu-$\saashort$ \\
    Wilcoxon: [n] z (p) & [163] -1.728 (0.084) & [233] -0.083 (0.934) & [183] -7.137 (*) & [232] -0.267 (0.789) & [182] -6.667 (*) & [234] -7.974 (*) \\
    proportion: [n] z (p) & [3] 0.0 (0.248) & [233] 0.7 (*) & [183] 0.891 (*) & [232] 0.69 (*) & [182] 0.879 (*) & [234] 0.726 (*) \\  \hline
    ubo50 & \stnu-$\pro_{0.9}$ & \stnu-$\react$ & \stnu-$\saashort$  & $\pro_{0.9}$-$\react$ & $\pro_{0.9}$-$\saashort$ & $\react$-$\saashort$ \\
  Wilcoxon: [n] z (p)   & [284] -4.827 (*) & [284] -5.557 (*) & [285] -10.588 (*) & [162] -2.643 (*) & [173] -8.741 (*) & [171] -7.484 (*) \\
   proportion: [n] z (p)  & [284] 0.454 (0.138) & [284] 0.479 (0.514) & [285] 0.712 (*) & [7] 1.0 (*) & [173] 0.936 (*) & [171] 0.906 (*) \\  \hline
    ubo100 & \stnu-$\pro_{0.9}$ & \stnu-$\react$ & \stnu-$\saashort$ & $\pro_{0.9}$-$\react$ & $\pro_{0.9}$-$\saashort$ & $\react$-$\saashort$ \\
      Wilcoxon: [n] z (p) & [288] -8.783 (*) & [285] -10.58 (*) & [284] -13.068 (*) & [122] -4.786 (*) & [146] -5.008 (*) & [141] -1.441 (0.15) \\
    proportion: [n] z (p)  & [288] 0.576 (*) & [285] 0.653 (*) & [284] 0.82 (*) & [23] 1.0 (*) & [146] 0.836 (*) & [141] 0.702 (*) \\
    \hline
    \end{tabular}}
    \caption{Pairwise comparison on time offline for noise factor $\epsilon=2$. Using a Wilcoxon test and a proportion test. Including all instances for which at least one of the two methods found a feasible solution. Note that the ordering matters: the first method showed is the better of the two in each pair method 1 - method 2 according to Wilcoxon. Each cell shows on the first row [nr pairs] the z-value (p-value) of the Wilcoxon test with (*) for p  $<$ 0.05. Each cell shows on the second row [nr pairs] the ratio of wins (p-value) with (*) for p $<$ 0.05.}
    \label{tab:offline_pairwise_2}
\end{table*}

\begin{table*}[htb!]
    \centering
     \scalebox{0.75}{
    \begin{tabular}{l|llllll} \hline
    j10 & $\pro_{0.9}$-$\saashort$ & $\pro_{0.9}$-$\stnu$ & $\saashort$-$\stnu$ & $\pro_{0.9}$-$\react$ & $\saashort$-$\react$ & \stnu-$\react$ \\
    Wilcoxon: [n] z (p)  & [340] -0.368 (0.713) & [340] -15.993 (*) & [340] -15.997 (*) & [340] -15.98 (*) & [340] -15.98 (*) & [340] -2.726 (*) \\
    proportion: [n] z (p) & [85] 0.471 (0.664) & [340] 1.0 (*) & [340] 1.0 (*) & [340] 1.0 (*) & [340] 1.0 (*) & [340] 0.765 (*) \\  \hline
    j20 & $\pro_{0.9}$-$\saashort$ & $\pro_{0.9}$-$\stnu$ & $\saashort$-$\stnu$ & $\pro_{0.9}$-$\react$ & $\saashort$-$\react$ & \stnu-$\react$ \\
    Wilcoxon: [n] z (p) & [270] -1.17 (0.242) & [270] -14.246 (*) & [270] -14.245 (*) & [270] -14.243 (*) & [270] -14.243 (*) & [270] -6.441 (*) \\
    proportion: [n] z (p) & [133] 0.556 (0.225) & [270] 1.0 (*) & [270] 1.0 (*) & [270] 1.0 (*) & [270] 1.0 (*) & [270] 0.852 (*) \\  \hline
    j30 & $\pro_{0.9}$-$\saashort$ & $\pro_{0.9}$-$\stnu$ & $\saashort$-$\stnu$ & $\pro_{0.9}$-$\react$ & $\saashort$-$\react$ & \stnu-$\react$ \\
    Wilcoxon: [n] z (p) & [320] -0.653 (0.514) & [320] -15.505 (*) & [320] -15.505 (*) & [320] -15.504 (*) & [320] -15.504 (*) & [320] -8.247 (*) \\
    proportion: [n] z (p) & [191] 0.524 (0.563) & [320] 1.0 (*) & [320] 1.0 (*) & [320] 1.0 (*) & [320] 1.0 (*) & [320] 0.875 (*) \\  \hline
    ubo50 & $\saashort$-$\pro_{0.9}$ & $\pro_{0.9}$-$\stnu$ & $\saashort$-$\stnu$ & $\pro_{0.9}$-$\react$ & $\saashort$-$\react$ & \stnu-$\react$ \\
    Wilcoxon: [n] z (p) & [370] -0.771 (0.44) & [370] -16.67 (*) & [370] -16.67 (*) & [370] -16.67 (*) & [370] -16.67 (*) & [370] -9.859 (*) \\
proportion: [n] z (p)     & [299] 0.522 (0.488) & [370] 1.0 (*) & [370] 1.0 (*) & [370] 1.0 (*) & [370] 1.0 (*) & [370] 0.892 (*) \\  \hline
    ubo100 & $\saashort$-$\pro_{0.9}$ & $\pro_{0.9}$-$\stnu$ & $\saashort$-$\stnu$ & $\pro_{0.9}$-$\react$ & $\saashort$-$\react$ & \stnu-$\react$ \\
    Wilcoxon: [n] z (p) & [400] -0.719 (0.472) & [400] -17.331 (*) & [400] -17.331 (*) & [400] -17.331 (*) & [400] -17.331 (*) & [400] -12.488 (*) \\
     proportion: [n] z (p) & [391] 0.453 (0.069) & [400] 1.0 (*) & [400] 1.0 (*) & [400] 1.0 (*) & [400] 1.0 (*) & [400] 0.925 (*) \\
    \hline
    \end{tabular}}
    \caption{Pairwise comparison on time online for noise factor $\epsilon=1$. Using a Wilcoxon test and a proportion test. Including all instances for which at least one of the two methods found a feasible solution. Note that the ordering matters: the first method showed is the better of the two in each pair method 1 - method 2. Each cell shows on the first row [nr pairs] the z-value (p-value) of the Wilcoxon test with (*) for p  $<$ 0.05. Each cell shows on the second row [nr pairs] proportion (p-value) with (*) for p $<$ 0.05.}
    \label{tab:online_pairwise_1}
\end{table*}

\begin{table*}[htb!]
    \centering
     \scalebox{0.75}{
    \begin{tabular}{l|llllll} \hline
    j10 & $\pro_{0.9}$-$\saashort$ & $\pro_{0.9}$-$\stnu$ & $\saashort$-$\stnu$ & $\pro_{0.9}$-$\react$ & $\saashort$-$\react$ & \stnu-$\react$ \\
    Wilcoxon: [n] z (p) & [235] -0.384 (0.701) & [258] -9.554 (*) & [258] -9.212 (*) & [234] -13.262 (*) & [235] -12.623 (*) & [258] -8.801 (*) \\
     proportion: [n] z (p) & [66] 0.47 (0.712) & [258] 0.907 (*) & [258] 0.899 (*) & [234] 1.0 (*) & [235] 0.987 (*) & [258] 0.891 (*) \\  \hline
    j20 & $\pro_{0.9}$-$\saashort$ & $\pro_{0.9}$-$\stnu$ & $\saashort$-$\stnu$ & $\pro_{0.9}$-$\react$ & $\saashort$-$\react$ & \stnu-$\react$ \\
    Wilcoxon: [n] z (p) & [182] -1.204 (0.229) & [216] -4.456 (*) & [217] -4.913 (*) & [177] -11.279 (*) & [182] -10.957 (*) & [217] -11.457 (*) \\
    proportion: [n] z (p) & [79] 0.43 (0.261) & [216] 0.815 (*) & [217] 0.825 (*) & [177] 0.994 (*) & [182] 0.984 (*) & [217] 0.968 (*) \\  \hline
    j30 & $\saashort$-$\pro_{0.9}$ & $\pro_{0.9}$-$\stnu$ & $\saashort$-$\stnu$ & $\pro_{0.9}$-$\react$ & $\saashort$-$\react$ & \stnu-$\react$ \\
    Wilcoxon: [n] z (p) & [182] -2.014 (*) & [232] -0.928 (0.353) & [234] -2.526 (*) & [163] -10.268 (*) & [183] -9.31 (*) & [233] -9.839 (*) \\
     proportion: [n] z (p) & [119] 0.58 (0.099) & [232] 0.69 (*) & [234] 0.735 (*) & [163] 0.982 (*) & [183] 0.94 (*) & [233] 0.901 (*) \\  \hline
    ubo50 & $\pro_{0.9}$-$\saashort$ & \stnu-$\pro_{0.9}$ & \stnu-$\saashort$ & $\pro_{0.9}$-$\react$ & $\saashort$-$\react$ & \stnu-$\react$ \\
    Wilcoxon: [n] z (p) & [173] -0.722 (0.47) & [284] -3.843 (*) & [285] -3.933 (*) & [162] -11.04 (*) & [171] -9.02 (*) & [284] -13.107 (*) \\
     proportion: [n] z (p) & [147] 0.524 (0.621) & [284] 0.43 (*) & [285] 0.435 (*) & [162] 1.0 (*) & [171] 0.942 (*) & [284] 0.951 (*) \\  \hline
    ubo100 & $\saashort$-$\pro_{0.9}$ & \stnu-$\pro_{0.9}$ & \stnu-$\saashort$ & $\pro_{0.9}$-$\react$ & $\saashort$-$\react$ & \stnu-$\react$ \\
     Wilcoxon: [n] z (p) & [146] -0.482 (0.63) & [288] -8.467 (*) & [284] -9.06 (*) & [122] -9.585 (*) & [141] -6.136 (*) & [285] -14.285 (*) \\
    proportion: [n] z (p)  & [144] 0.458 (0.359) & [288] 0.576 (*) & [284] 0.585 (*) & [122] 1.0 (*) & [141] 0.837 (*) & [285] 0.982 (*) \\
    \hline
    \end{tabular}}
    \caption{Pairwise comparison on time online for noise factor $\epsilon=2$. Using a Wilcoxon test and a proportion test. Including all instances for which at least one of the two methods found a feasible solution. Note that the ordering matters: the first method showed is the better of the two in each pair method 1 - method 2. Each cell shows on the first row [nr pairs] the z-value (p-value) of the Wilcoxon test with (*) for p  $<$ 0.05. Each cell shows on the second row [nr pairs] proportion (p-value) with (*) for p $<$ 0.05.}
    \label{tab:online_pairwise_2}
\end{table*}

\end{document}